\documentclass{article}

\usepackage[preprint]{neurips_2020}

\usepackage{booktabs}       
\usepackage{multirow}       
\usepackage{nicefrac}       
\usepackage{microtype}      
\usepackage{graphicx,subcaption}  
\usepackage{amsfonts}       
\usepackage{amssymb}        
\usepackage[mathscr]{euscript} 
\usepackage{amsthm}         
\usepackage{bm}             
\usepackage{dsfont}         
\usepackage{mathrsfs}       
\usepackage{mathtools}      
\usepackage{wrapfig}
\usepackage{url}
\usepackage{xcolor}
\usepackage{accents}       

\usepackage{enumitem}      

\theoremstyle{plain}
\newtheorem{theorem}{Theorem}

\newtheorem*{theorem*}{Theorem}
\newtheorem*{lemma*}{Lemma}
\newtheorem{lemma}[theorem]{Lemma}

\newtheorem{remark}{Remark}

\theoremstyle{definition}

\DeclareMathOperator*{\subjto}{s.t.}

\newcommand{\set}[1]{\mathcal{#1}}
\newcommand{\vect}[1]{\bm{#1}}

\newcommand{\whvect}[1]{\widehat{\bm{#1}}}

\DeclarePairedDelimiter{\paren}{(}{)}                       

\newcommand{\diag}[1]{\operatorname{diag}\paren*{#1}}       
\newcommand{\Ind}[1]{\mathds{I}\paren*{#1}}                 

\def  \Tr    {\mathsf{T}}                       

\def \BigO   {\mathcal{O}}

\makeatletter
\newcommand*\dotp{\mathpalette\dotp@{0.55}}
\newcommand*\dotp@[2]{\mathbin{\vcenter{\hbox{\scalebox{#2}{$\m@th#1\bullet$}}}}}
\makeatother

\usepackage{stmaryrd} 

\DeclareMathSymbol{\widetildesym}{\mathord}{largesymbols}{"65}
\newcommand\lowerwidetildesym{%
	\text{\smash{\raisebox{-1.3ex}{%
				\hspace{.3ex}$\widetildesym$}}}}
\newcommand\fixwidetilde[1]{%
	\mathchoice
	{\accentset{\displaystyle\lowerwidetildesym}{#1}}
	{\accentset{\textstyle\lowerwidetildesym}{#1}}
	{\accentset{\scriptstyle\lowerwidetildesym}{#1}}
	{\accentset{\scriptscriptstyle\lowerwidetildesym}{#1}}
}
\newcommand{\wtvect}[1]{\fixwidetilde{\bm{#1}}}

\usepackage{overpic}
\usepackage[ruled,vlined,noend]{algorithm2e}
\newcommand{\Function}[1]{\normalfont\textsc{#1}}
\SetKwProg{Fn}{Function}{:}{}
\SetAlFnt{\small}

\definecolor{darkyellow}{rgb}{0.64, 0.48, 0}
\definecolor{darkgreen}{rgb}{0, 0.5, 0.22}
\definecolor{darkred}{rgb}{0.75, 0, 0}

\title{Restructuring, Pruning, and Adjustment of Deep Models for Parallel Distributed Inference}
\author{
	Afshin Abdi, Saeed Rashidi, Faramarz Fekri, Tushar Krishna \\
	School of Electrical and Computer Engineering \\
	Georgia Institute of Technology, Atlanta, GA \\
	\texttt{\{abdi,saeed.rashidi,fekri,tushar\}@gatech.edu}
}

\begin{document}
	
\maketitle

\begin{abstract}
Using multiple nodes and parallel computing algorithms has become a principal tool to improve training and execution times of deep neural networks as well as effective collective intelligence in sensor networks.
In this paper, we consider the parallel implementation of an already-trained deep model on multiple processing nodes (a.k.a. workers) where the deep model is divided into several parallel sub-models, each of which is executed by a worker. Since latency due to synchronization and data transfer among workers negatively impacts the performance of the parallel implementation, it is desirable to have minimum interdependency among parallel sub-models. To achieve this goal, we propose to rearrange the neurons in the neural network and partition them (without changing the general topology of the neural network), such that the interdependency among sub-models is minimized under the computations and communications constraints of the workers. We propose RePurpose, a layer-wise model restructuring and pruning technique that guarantees the performance of the overall parallelized model. To efficiently apply RePurpose, we propose an approach based on $\ell_0$ optimization and the Munkres assignment algorithm. We show that, compared to the existing methods, RePurpose significantly improves the efficiency of the distributed inference via parallel implementation, both in terms of communication and computational complexity.

\end{abstract}

\section{Introduction}
\label{sec:introduction}
In recent years, the size and complexity of deep neural networks has been increased significantly in terms of model's structure and number of parameters. 
Consequently, real-time implementation and inference in many machine learning (ML) problems has become a challenging task. 
Although the execution time of deep neural networks can be improved significantly by the application of parallel computing algorithms and using multiple processing units (such as GPU's or clusters of computing nodes), it generally requires synchronization and data exchange among processing units to some extent. 
%
This is mainly due to the fact that in parallel computations, each processing unit performs a portion of the computations, its inputs generally depend on the outputs from other units, and the results of computations should be aggregated to yield the desired output. These co-dependencies can lead to significant delays in computations.
%
Moreover, in some real-world scenarios, such as sensor networks, the inference is done on the data observed by the entire network, i.e., each node in the network only observes part of the data. 
However, transferring all data to a central powerful node to aggregate and perform the ML task is undesirable due to the sheer amount of data to be collected, limited computational power, privacy concerns, or even availability of such a node. 
Hence, it is more favorable to develop a distributed equivalence of a deep model for deploying over the processors/sensor network. 

In the aforementioned applications, straightforward parallel computing algorithms cannot be arbitrarily scaled up for deep models with complex connectivity structures.
The majority of past works on distributed/parallel execution of deep neural networks are concerned with algorithmic aspects of the parallel implementation of the neural network (e.g., \cite{Zinkevich2010Parallelized,Chung2014Parallel,DeGrazia2012Parallelization}). 
However, here, we focus on the structure of deep models and how we can modify it for efficient parallel distributed implementation.

%
In recent years, there has been an increasing interest in compressing, pruning, or modifying the structure of deep models to reduce their computational or storage costs, while keeping the accuracy or performance of the modified model acceptable. 
The majority of these approaches can be classified into three categories:
\begin{itemize}[leftmargin=12pt,itemsep=1pt,partopsep=1pt,parsep=1pt,topsep=1pt]
	\item \emph{Knowledge Distillation} to train a shallow or small model (referred to as student network) that mimics the behavior of an already trained complex model (a.k.a. teacher network) or an ensemble of teacher networks (see e.g., \cite{Hinton2015Distilling,Romero2015Fitnets,Zagoruyko2017Paying}). 
	
	\item \emph{Using Structured Parameters} to reduce the size of deep model or its processing time. 
	Examples include using circulant matrices~\cite{Cheng2015Fast} or Adaptive Fastfood transform~\cite{Yang2015Deep} for fully connected layers, and separable filters ~\cite{Rigamonti2013Learning} or low-rank tensor decomposition~\cite{Tai2016Convolutional} for convolutional layers.
	
	\item \emph{Pruning Parameters} has been used extensively to reduce the complexity of the model as well as over-parametrization. $\ell_1$ or $\ell_0$ regularization~\cite{Louizos2018Learning}, and group-sparsity~\cite{Zhou2016Less,Wen2016Learning} have been successfully used to promote sparsity of the parameters during training. 
	Model pruning algorithms such as Optimal Brain Damage \cite{Cun1990Optimal}, Optimal Brain Surgeon \cite{Hassibi1993Second}, hard-thresholding the parameters \cite{Han2015Learning}, and similar works \cite{Castellano1997iterative,Leung2001pruning}, mainly focus on removing the insignificant edges or neurons, by considering the magnitude of the weights or their approximate Hessian matrix as a measure of importance. More recently, Aghasi, et al.~\cite{Aghasi2017Net,Aghasi2020Fast} proposed Net-Trim, a convex optimization technique to prune the parameters of the deep model by analyzing the signals in the neural network.
\end{itemize}

Although it is possible to design deep models according to the capability and constraints of the processing system, following such an approach requires training a new deep model for every target hardware which is infeasible or demanding in many ML problems. 
Further, imposing a possibly unnecessary structure in advance during training a deep model would likely be limiting in terms of model performance and accuracy. It will be also an undesirable approach for parallel implementation since a model specifically designed for optimum implementation on a target platform or architecture may be far from optimum on other platforms (e.g., GPUs with different compute capabilities, or CPU vs GPU vs sensor network). Hence optimizing and fixing the structure for one particular parallel distributed setting in advance would limit the optimal deployment on other platforms.
As a result, we assume that a complex deep model has already been trained with minimum or no hardware-specific constraints on its parameters or structure. Our goal would be readjusting the model via restructuring the layers and manipulating the parameters of the neural network without changing its general topology for more efficient parallel implementation. 
%
%

\begin{wrapfigure}[11]{r}{0.5\textwidth}
	\vspace{-3mm}
	\centering
	\begin{overpic}[width=\linewidth]{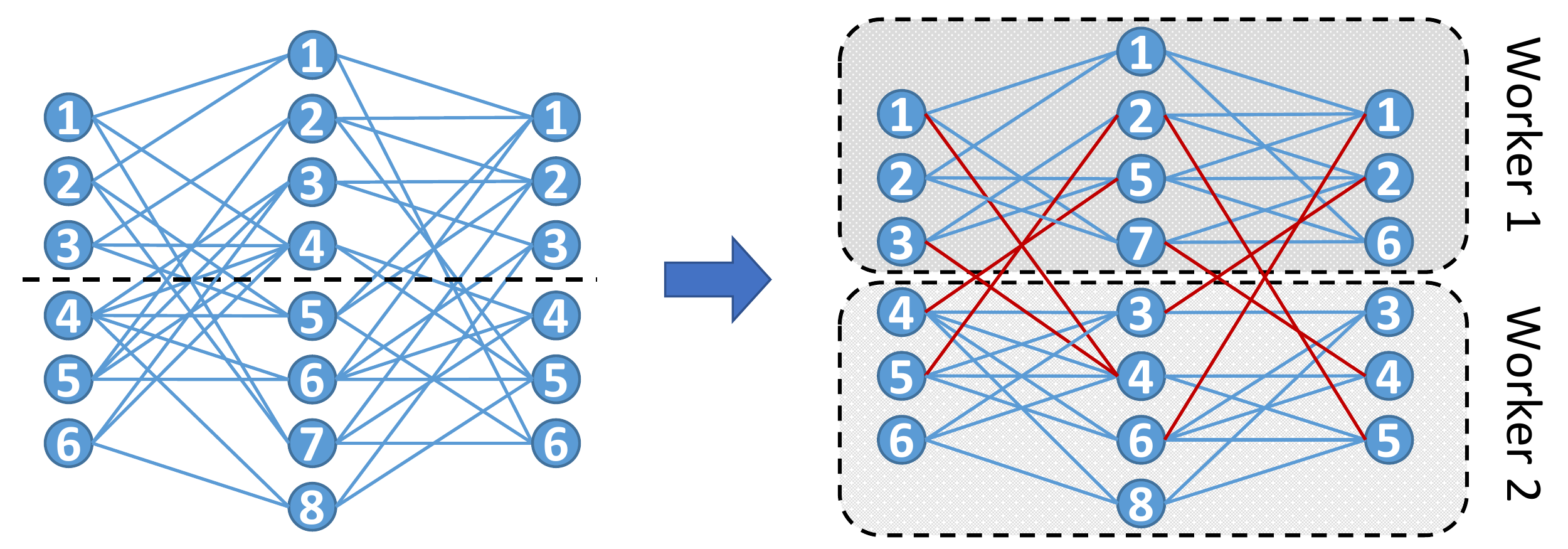}
		\put (4,-5) {\scalebox{.9}{(a) Original model}} 
		\put (52,-5) {\scalebox{.9}{(b) Restructured model}} 
	\end{overpic}
	~
	\caption{Restructuring a neural network to reduce communication between processing units}
	\label{fig:RestructuringModelExample}
	\vspace{-5mm}
\end{wrapfigure}
As an example, consider the simple neural network in Fig.~\ref{fig:RestructuringModelExample}(a). Simply partitioning the model into two sub-models (as depicted by a dashed line in the Fig.~\ref{fig:RestructuringModelExample}(a)) imposes lots of communication between the two partitions. However, by rearranging the neurons properly, the co-dependency (and hence required communications) between the two sub-models (the red edges in Fig.~\ref{fig:RestructuringModelExample}(b)) is reduced substantially. It is worth mentioning that there are approximately $\BigO(P^N)$ different partitioning to distribute computations of a neural network's layer with $N$ neurons over $P$ workers. Hence, enumerating all such possibilities and choosing a good one is infeasible specially for large networks. In this paper, we propose a systematic approach to perform such partitioning and parameter adjustment to ensure efficient implementation of the modified model while keeping its accuracy close to the original model.

%
\vspace{4mm}
\subsection*{Notations}
\begin{wrapfigure}[6]{r}{0.3 \linewidth}
	\vspace{-4mm}
\begin{tabular}{|l|l|} 
	\toprule
	$P$ & Number of workers \\
	$N$ & Number of neurons \\
	$L$ & Number of layers \\
	$\vect{W}$ & Weight matrix \\
	$\vect{b} $ & Bias \\
	$\vect{M}$ & Mask matrix \\
	\bottomrule 
\end{tabular}
\end{wrapfigure}
Bold lowercase letters represent vectors and the $i$-th element of the vector $\vect{x}$ is denoted as $x_i$. Matrices are denoted by bold capital letters such as $\vect{X}$, with the $(i,j)$-th element represented by $X_{i,j}$ or $[\vect{X}]_{i,j}$. 
$\vect{A}\odot\vect{B}$ is the Hadamard (element-wise) product of $\vect{A}$ and $\vect{B}$. $\|\vect{X}\|_F$ is the Frobenius norm of $\vect{X}$, $\|\vect{x}\|_2$ and $\|\vect{x}\|_0$ are the $\ell_2$ and $\ell_0$ norms of $\vect{x}$, respectively.
$\vect{1}$ is a vector or matrix of all ones, whose size would be clear from the context.

\section{Problem Statement and our Approach}
\label{sec:problem_statement}
Consider the problem of parallel distributed implementation of a trained deep neural network over $P$ processing units (hereafter referred to as workers), where the deep model is divided into $P$ sub-models, each of which is executed by a worker. As managing the synchronization and data transfer among workers degrades the efficiency of the parallel implementation (e.g., higher latency), it is crucial to reduce the communication among workers. 
The communication is needed between the workers when the input of a neuron in a sub-model is from the output of a neuron belonging to a different sub-model which resides in another worker. These co-dependencies can lead to significant delays in computation.

For the sake of simplicity in presentations and analysis, here, we mainly focus on feedforward deep models, specifically fully-connected layers. Note that the convolution layer can be represented as a special case of a fully connected layer. \footnote{Recall that the convolution $\vect{h}*\vect{x}$ can be represented as $\vect{W}^\Tr\vect{x}$ for a circulant matrix $\vect{W}$ constructed from $\vect{h}$.} For more details and the extensions of our approach to other complex architectures, please refer to the supplementary document.

Consider an arbitrary neural network with $L$ layers and parameters $\{\vect{\theta}^{(l)}\}_{l=1}^L$, where $\vect{\theta}^{(l)}=\{\vect{W}^{(l)},\vect{b}^{(l)}\}$ is the parameters of the $l$-th layer. Let $\vect{x}^{(l)}$ be the input signal to the $l$-th layer. Then, the output of the layer (input to the next layer) would be given by
\begin{equation}
	\vect{y}^{(l)} = (\vect{W}^{(l)})^\Tr\vect{x}^{(l)}+\vect{b}^{(l)},\quad\quad\vect{x}^{(l+1)}=\sigma(\vect{y}^{(l)}),
\end{equation}
where $\sigma(\cdot)$ is the activation function.

\begin{wrapfigure}[15]{r}{0.5\textwidth}
	\vspace{-6mm}
	\centering
	\begin{overpic}[width=0.9\linewidth]{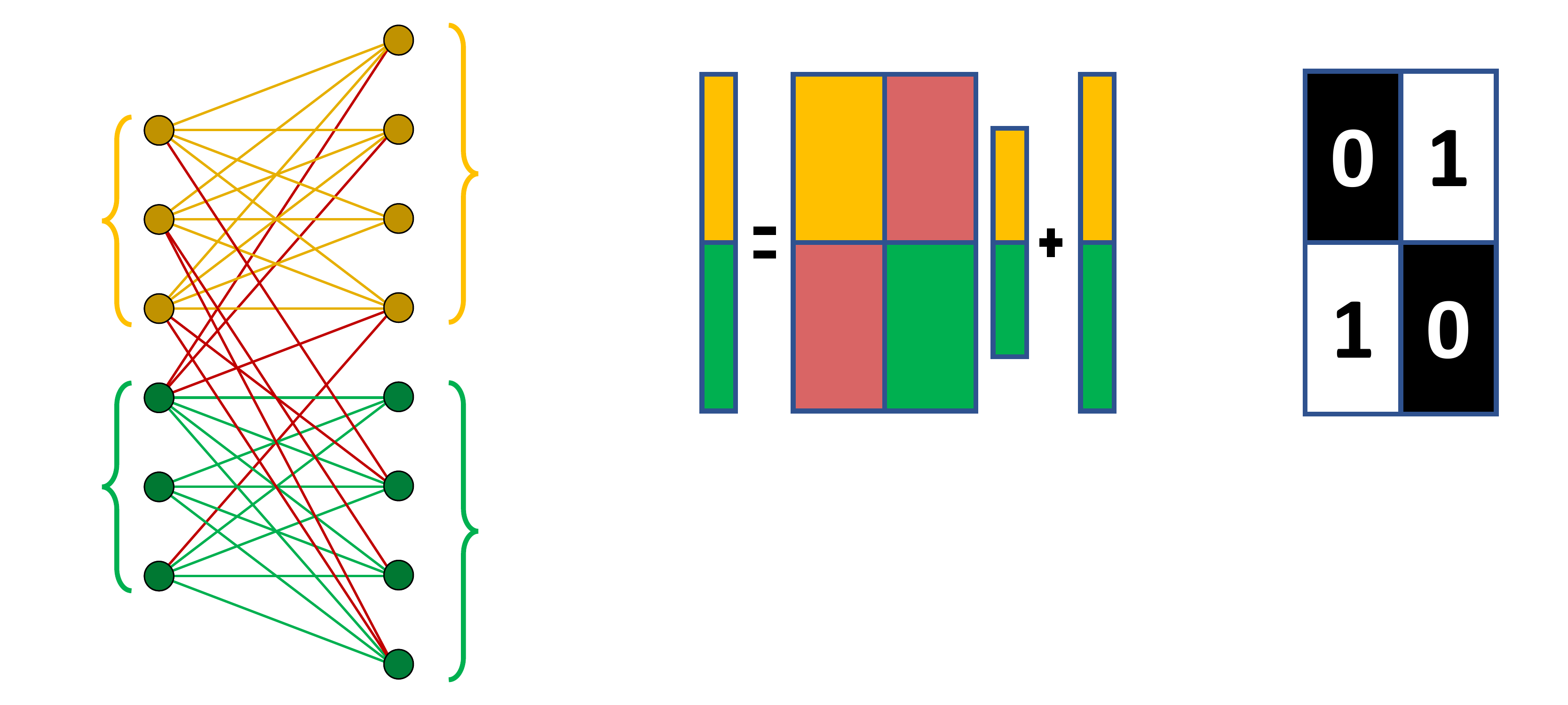}
		\put (0, 30) {\scalebox{.9}{$\vect{x}_1$}} 
		\put (0, 13) {\scalebox{.9}{$\vect{x}_2$}} 
		\put (32, 33) {\scalebox{.9}{$\vect{y}_1$}} 
		\put (32, 10) {\scalebox{.9}{$\vect{y}_2$}} 
		\put (45, 10) {\scalebox{.8}{${\color{darkyellow}\vect{y}_1}={\color{darkyellow}\vect{W}_{11}^\Tr\vect{x}_1+\vect{b}_1}+{\color{darkred}\vect{W}_{12}^\Tr\vect{x}_2}$}} 
		\put (45, 2) {\scalebox{.8}{${\color{darkgreen}\vect{y}_2}={\color{darkgreen}\vect{W}_{22}^\Tr\vect{x}_2+\vect{b}_2}+{\color{darkred}\vect{W}_{21}^\Tr\vect{x}_1}$}} 
	\end{overpic}
	\caption{Communication between workers in parallel execution of a model over two workers. The intra-worker computations are denoted by \textcolor{darkyellow}{yellow} and \textcolor{darkgreen}{green} connections, while required communication between the workers are denoted by \textcolor{darkred}{red} edges. The binary mask matrix (right image) can be used to determine the edges between the two workers.}
	\label{fig:problem_formulation}
	\vspace{-5mm}
\end{wrapfigure}
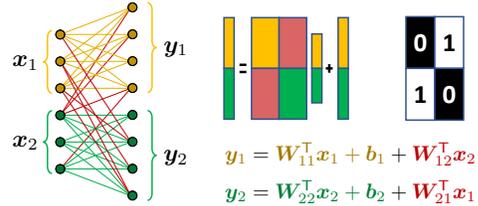
To analyze the bottlenecks, consider an arbitrary layer with input $\vect{x}$, and parameters $\vect{W}$ and $\vect{b}$ (Fig.~\ref{fig:problem_formulation}). 
Hence, $\vect{y}=\vect{W}^\Tr\vect{x}+\vect{b}$ would be the input signal to the neurons of the layer.
Suppose that $\vect{x}_k$ and $\vect{y}_k$ are subsets of the signals that are processed by the $k$-th worker. Without loss of generality, we assume that the neurons are ordered such that the $k$-th block of consecutive neurons belongs to the $k$-th sub-model, i.e., $\vect{x}=[\vect{x}_1;\vect{x}_2;\ldots;\vect{x}_P]$. 
By partitioning $\vect{W}$ and $\vect{b}$ accordingly, we observe that
\begin{equation}
	\vect{y}_k = (\vect{W}_{k,k}^\Tr\vect{x}_k + \vect{b}_k) + {\color{darkred}(\sum_{l\neq k} \vect{W}_{k,l}^\Tr\vect{x}_l)}.
\end{equation}
Note that the first term can be computed at the $k$-th worker independent of the others, whereas computing the second term requires synchronization and communication from the other workers. Hence, to reduce the dependency among workers and the communication cost, we consider minimizing the number of non-zero elements in $\vect{W}_{k,l}$, for $l\neq k$. 

By defining an appropriate binary mask $\vect{M}$ (Fig.~\ref{fig:problem_formulation} (right)), the connections between sub-models  can be determined by the non-zero elements of $\vect{M}\odot\vect{W}$.
In general, if $\iota_k$ and $o_k$ are the number of input and output neurons assigned to the $k$-th worker, then 
$\vect{M}$ is an anti-diagonal block matrix, given by
\begin{equation*}
	\vect{M}=1-\diag{\vect{1}_{\iota_{_1}\times o_{_1}},\ldots,\vect{1}_{\iota_{_P}\times o_{_P}}}.
\end{equation*}

\begin{remark}
	Note that the bias $\vect{b}$ does not contribute to the communication between workers and can be safely ignored in computing the cost. 
	Further, $\|\vect{M}\odot\vect{W}\|_0$ can be viewed as the number of edges between sub-models, and be used as an approximation to the latency caused by the communication and synchronization among workers. 
	Similarly, by defining an appropriate binary mask $\vect{M}_{ij}$, we can find the edges from worker $j$ to $i$ from the non-zero entries of $\vect{V}_{ij}:=\vect{M}_{ij}\odot\vect{W}$. Depending on the communication protocol among workers, the number of non-zero edges, number of non-zero rows, or number of non-zero columns of $\vect{V}_{ij}$ can be interpreted as a measure of latency due to the communication from worker $j$ to $i$. 
	For the sake of simplicity, in this work, we consider $\|\vect{M}\odot\vect{W}\|_0$ as a measure of total communication latency. However, the extensions of our proposed approach to other cases is straightforward.
\end{remark}

To reduce the communication, one may attempt to reduce the number of cross-edges among sub-models. However, as we observed in our experiments, generally there are many \emph{important} connections between neurons from different sub-models, and removing those connections can severely affect the performance of the neural network. Hence, it is important to have such neurons in the same sub-model. On the other hand, the problem of neuron assignment to the workers is combinatorial and discrete with complexity $\BigO(P^N)$ for a layer with $N$ neurons and $P$ workers. Hence, enumerating all possibilities or using ordinary optimization techniques as well as genetic algorithms or simulated annealing would fail due to the complex nature of interactions among neurons in a deep neural network. 
Based on the above observations, we devise \emph{RePurpose}, a layer-wise neural network restructuring and pruning for efficient parallel implementation. The gist of the idea is as follows; 

\vspace{1mm}
\centerline{\fbox{
\begin{minipage}{0.96\linewidth}
The neurons of the input layer are assigned to the sub-models based on each worker's computational power and/or structure of the input data. For example, in a sensor network, it is dictated by the input of each sensor. 
Next, we restructure and adjust the neural network, sequentially one layer at a time. 
For the $l$-th layer, the assignments of the neurons in layer $l-1$ are assumed to be fixed and known from the previous steps. The neurons in layer $l$ are rearranged and assigned to each sub-model, and the parameters of the layer are pruned and fine-tuned, 
such that ($i$) the performance of the modified neural network is close to the original one, and ($ii$) the communication between the  sub-models (measured by the number of edges connecting neurons from different sub-models) is minimized. 
\end{minipage}
}}

\section{RePurpose: Restructuring and Pruning Deep Models}
\label{sec:repurpose}
\begin{wrapfigure}[11]{4}{0.45\textwidth}
\vspace{-5mm}
	\centering
	\includegraphics[width=\linewidth]{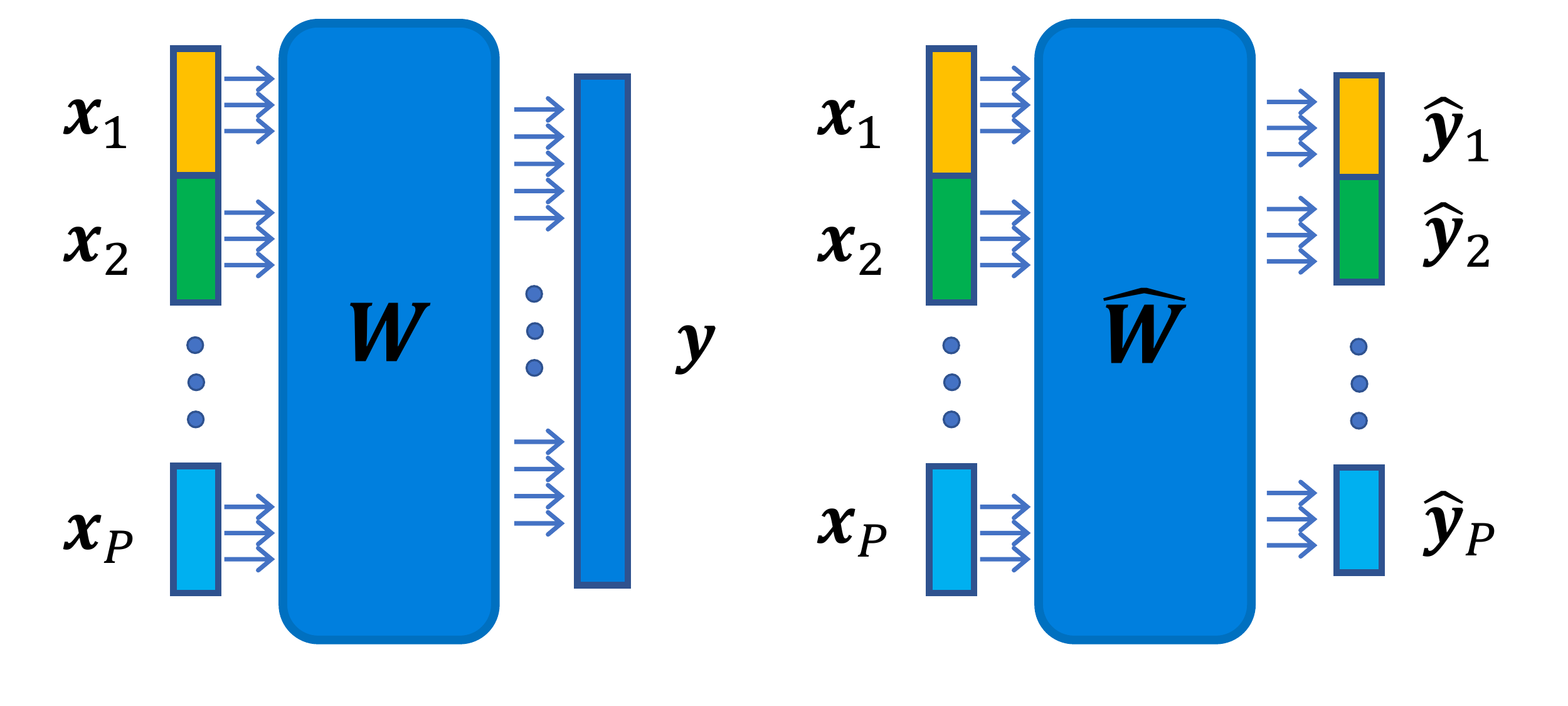}
	\vspace{-5mm}
\caption{Rearranging neurons of a layer and adjusting parameters such that the $k$-th block of signals, $\whvect{y}_k$, is processed at the $k$-th worker.}
\label{fig:neuron_rearrangement}
\vspace{-5mm}	
\end{wrapfigure}
Consider the $l$-th layer of neural network and assume that the neurons in the previous layers have already been partitioned and rearranged, i.e., the input of the layer is partitioned as $[\vect{x}_1;\ldots;\vect{x}_P]$, where $\vect{x}_k$ is computed at the $k$-th worker. Let $\vect{y}$ and $\vect{W}$ be the signals and parameters of the $l$-th layer in the original model. RePurpose rearranges the neurons such that the $k$-th block of neurons are being assigned to the $k$-th worker (Fig.~\ref{fig:neuron_rearrangement}). Note that the rearrangement of the neurons can be captured via a permutation matrix $\vect{\Pi}$. Hence, if we use the same weights, the effect of neuron-rearrangement can be formulated as $\whvect{y}=\vect{\Pi y}$ and $\whvect{W}=\vect{W\Pi}^\Tr$, and the number of cross-edges between workers would be $\|\vect{M}\odot\whvect{W}\|_0$. To further reduce the communication between workers, RePurpose not only rearranges the neurons, but it also prunes and adjusts $\whvect{W}$. 
Hence, the optimization problem for RePurpose is formulated as 
\begin{equation}
	\min_{\whvect{W},\vect{\Pi}} \|\vect{M}\odot\whvect{W}\|_0\quad\quad \subjto~\|\whvect{W}-\vect{W\Pi}^\Tr\|_F^2\leq \epsilon, 
	\label{eqn:RePurposeP-OPT}
\end{equation}
where $\epsilon$ is a parameter controlling the closeness of the parameters. Directly solving \eqref{eqn:RePurposeP-OPT} is infeasible as it is (mixed-)discrete, non-convex, and there are $N!$ different permutation matrices. In the following, we propose an alternative and efficient approach to solve \eqref{eqn:RePurposeP-OPT}.

Recall that if neuron $i$ is assigned to worker $j$, the signal at that neuron can be rewritten as
$\hat{y}_i = b_i + \whvect{w}_i^\Tr\vect{x} = b_i + \whvect{w}_{ij}^\Tr\vect{x}_j + \sum_{k\neq j} \whvect{w}_{ik}^\Tr\vect{x}_k$, 
where $\whvect{w}_i$ is the $i$-th column of $\whvect{W}$, and $\whvect{w}_{ik}$ is the $k$-th block of $\whvect{w}_i$ corresponding to $\vect{x}_k$. Hence, the communication cost from other workers to worker $j$ would be $\|\whvect{w}_{i,\backslash j}\|_0:=\sum_{k\neq j} \|\whvect{w}_{ik}\|_0$. By incorporating an additional optional cost to encourage the total sparsity of the parameters, $\|\whvect{w}_i\|_0$, the cost of assigning neuron $i$ to worker $j$ would be
\begin{equation}
	c_{ji} = \min_{\whvect{w}_i} \|\vect{w}_i-\whvect{w}_i\|_2^2 + \eta_1 \|\whvect{w}_i\|_0 + \eta_2 \|\whvect{w}_{i,\backslash j}\|_0,
	\label{eqn:PRP_param_cost}
\end{equation}
where $\eta_1$ and $\eta_2$ control the trade-off between the error, sparsity, and cross-communication.

\begin{lemma} \label{lem:PRP_optimum_weight}
	The solution of \eqref{eqn:PRP_param_cost} is given by element-wise hard-thresholding $\vect{w}_i$, i.e.,
	\begin{equation}
	[\whvect{w}_i]_n = \left\{
	\begin{array}{ll}
	0 & |\left[\vect{w}_i\right]_n| \leq \sqrt{\eta} \\
	\left[\vect{w}_i\right]_n & \textup{o.w.}
	\end{array}
	\right.
	\label{eqn:PRP_optimum_weight}
	\end{equation}
	where $\eta=\eta_1$ or $\eta_1+\eta_2$, depending on whether neuron $n$ from the previous layer has been assigned to the $j$-th worker or not.
\end{lemma} 

\begin{wrapfigure}[10]{r}{0.48\textwidth}
	\vspace{-5mm}	
	\begin{algorithm}[H]
		\SetAlgoLined
		\caption{\small Parameter-Space RePurpose}
		\label{alg:RePurposeP}
		\DontPrintSemicolon
		
		\SetKwFunction{FMain}{\Function{RePurpose}}
		\Fn{\FMain{$\vect{W}$, $\{n_k\}_{k=1}^P$, $\eta_1$, $\eta_2$}}
		{
			- Compute the cost matrix $\vect{C}$, where $[\vect{C}]_{j,i}$ is calculated via \eqref{eqn:PRP_param_cost} and \eqref{eqn:PRP_optimum_weight}\;
			- Construct $\wtvect{C}$ by repeating the $k$-th row of $\vect{C}$, $n_k$ times.\;
			- $(I, J) = \Function{Munkres}(\wtvect{C})$\;
			- Define permutation matrix as $\vect{\Pi}_{I,J} = 1$\;
			\vspace{1mm}
			\textbf{Return} $\vect{\Pi}$\;
		}
	\end{algorithm}
\end{wrapfigure}
Restructuring and neuron assignment can be interpreted as selecting elements from the cost matrix $\vect{C}$, whose $(j,i)$-th element is given by \eqref{eqn:PRP_param_cost}, such that ($i$) from row $k$, $n_k$ elements are selected, i.e., $n_k$ neurons are assigned to worker $k$, ($ii$) from each column, only one element is selected, i.e., each neuron can be assigned to only one worker, and ($iii$) the sum of selected elements is minimized, i.e., the total cost of neuron assignment and parameter adjustment is minimum.

Algorithm~\ref{alg:RePurposeP} summarizes the proposed solution, where \Function{Munkres}($\cdot$) uses the Munkres assignment algorithm~\cite{Kuhn1955Hungarian,Munkres1957Algorithms} to find the (row-column) index pairs that minimizes the total sum cost $\sum_n [\wtvect{C}]_{_{I_n,J_n}}$. 

\begin{theorem}
	Algorithm \ref{alg:RePurposeP} finds the optimum solution of 
	\begin{equation}
		\|\whvect{W}-\vect{W\Pi}^\Tr\|_F^2 + \eta_1 \|\whvect{W}\|_0 + \eta_2\|\vect{M}\odot\whvect{W}\|_0,
		\label{eqn:Lagrangian_RePurposeP_OPT}
	\end{equation}
	with time complexity $\BigO(N^3)$, where $N$ is the number of layer's neurons (number of columns of $\vect{W}$).
\end{theorem}
Note that by setting $\eta_1=0$, \eqref{eqn:Lagrangian_RePurposeP_OPT} would be the Lagrangian of \eqref{eqn:RePurposeP-OPT} and choosing appropriate value for $\eta_2$ can lead to the desired error bound $\|\whvect{W}-\vect{W\Pi}^\Tr\|_F^2\leq \epsilon$. 
Finally, it is worth mentioning that the bias term does not contribute to the communication cost and is given by $\whvect{b}=\vect{\Pi b}$.

\begin{remark}
	In model pruning and compression, it is common to retrain the modified model to fine-tune the parameters and improve the accuracy of the model. This extra post-processing is generally referred to as post-training phase or fine-tuning. The same principle can be applied to our proposed algorithm.
\end{remark}

\section{Experiments}
\label{sec:experiments}
To evaluate the performance of the RePurpose algorithm, we consider different neural network architectures and compare the accuracy, communication and wall-clock times w.r.t. \emph{naive} implementation where the input data is communicated to all nodes in the network, so they all have the entire input data, \emph{baseline} which is parallel implementation of the deep model without any modification to the parameters or structures, and \emph{sparse} implementation which sparsifies the parameters to reduce cross-edges between the workers without re-arranging the neurons. We evaluate the accuracy-communication trade-off in different sensor networks, as well as the reduction in total computation time (wall-clock time) in Edge and Data Center platforms.

\subsection{Sensor Network}

\textbf{Setup 1.} Figure~\ref{fig:RePurpose-Spiral}(a) shows a 2 sensors network, sensor $i$ observes location $x_i$ of a target object and each sensor's task is to determine whether the object is in the blue or green region. A simple neural network (Fig.~\ref{fig:RePurpose-Spiral}(b)) is trained at a central node to perform the task with accuracy $94.5\%$. In the naive approach, the sensors exchange their observations ($x_i$'s) and run the inference (NN) independently. Hence, the NN is executed twice throughout the network at the cost of higher computational complexity. Alternatively, we can apply RePurpose to efficiently distribute the inference over the sensors.
We applied RePurpose with $\eta_1=0$, $\eta_2=0.01$ (Fig.~\ref{fig:RePurpose-Spiral}(c)), and $\eta_2=0.1$ (Fig.~\ref{fig:RePurpose-Spiral}(d)). As a result, the number of cross-worker communications reduced significantly to $1.7\%$, $1.5\%$ and $1.6\%$ for $\eta_2=0.01$, and $0.7\%$, $0.1\%$ and $0.3\%$ for $\eta_2=0.1$ for layers 1, 2, and 3, respectively. Specifically, with only $6$ communicated values, the computational complexity at each sensor is reduced by almost a factor of $4$ compared to the naive implementation. 
However, the accuracy of the distributed parallel model, prior to the post-training phase, is reduced to $93.5\%$. By retraining the modified model for few iterations (and imposing the structural constraints found through RePurpose), the accuracy of the fine-tuned model becomes $94.4\%$.

\begin{figure}[!htbp]
	\centering
	\begin{overpic}[width=\linewidth]{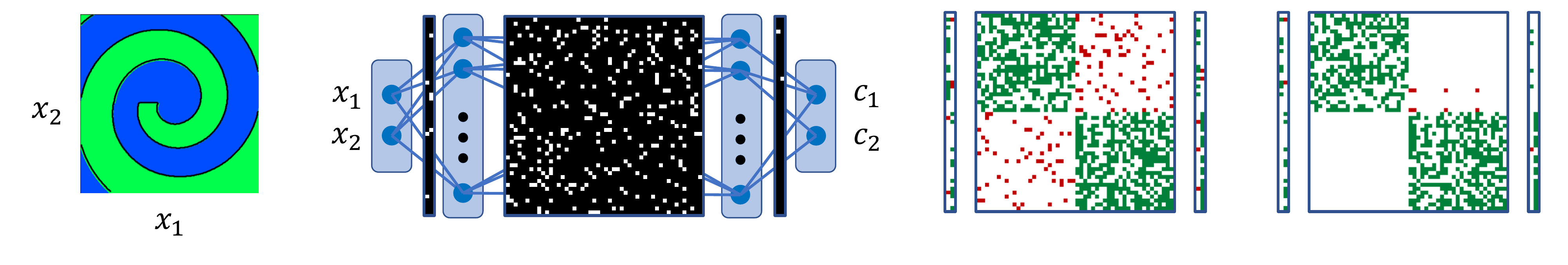}
		\put (0,-2) {\scalebox{.75}{(a) Classification Regions}}
		\put (28,-2) {\scalebox{.75}{(b) Original trained model}} 
		\put (60,-2) {\scalebox{.75}{(c) RePurpose $\eta_2=0.01$}} 
		\put (81,-2) {\scalebox{.75}{(d) RePurpose $\eta_2=0.1$}} 
		\put (26.2,1) {\scalebox{.6}{$\vect{W}_1$}} 
		\put (38,1) {\scalebox{.6}{$\vect{W}_2$}} 
		\put (49,1) {\scalebox{.6}{$\vect{W}_3$}} 
		\put (60,1) {\scalebox{.6}{$\whvect{W}_1$}} 
		\put (68,1) {\scalebox{.6}{$\whvect{W}_2$}} 
		\put (76,1) {\scalebox{.6}{$\whvect{W}_3$}} 
		\put (81,1) {\scalebox{.6}{$\whvect{W}_1$}} 
		\put (89,1) {\scalebox{.6}{$\whvect{W}_2$}} 
		\put (97,1) {\scalebox{.6}{$\whvect{W}_3$}} 
	\end{overpic}
	\vspace{0mm}
	\caption{\textbf{Setup 1.} Distributed inference over a sensor network to classify location of an object. The zero coefficients are represented by empty (white) spaces, inner-worker connection by \textcolor{darkgreen}{green} pixels and cross-worker edges by \textcolor{darkred}{red} pixels in the images. Note that for the illustration purposes, the coefficient matrix of the first layer is transposed.}
	\label{fig:RePurpose-Spiral}
	\vspace{2mm}
	\begin{subfigure}{0.32\linewidth}
		\centering
		\includegraphics[width=\linewidth]{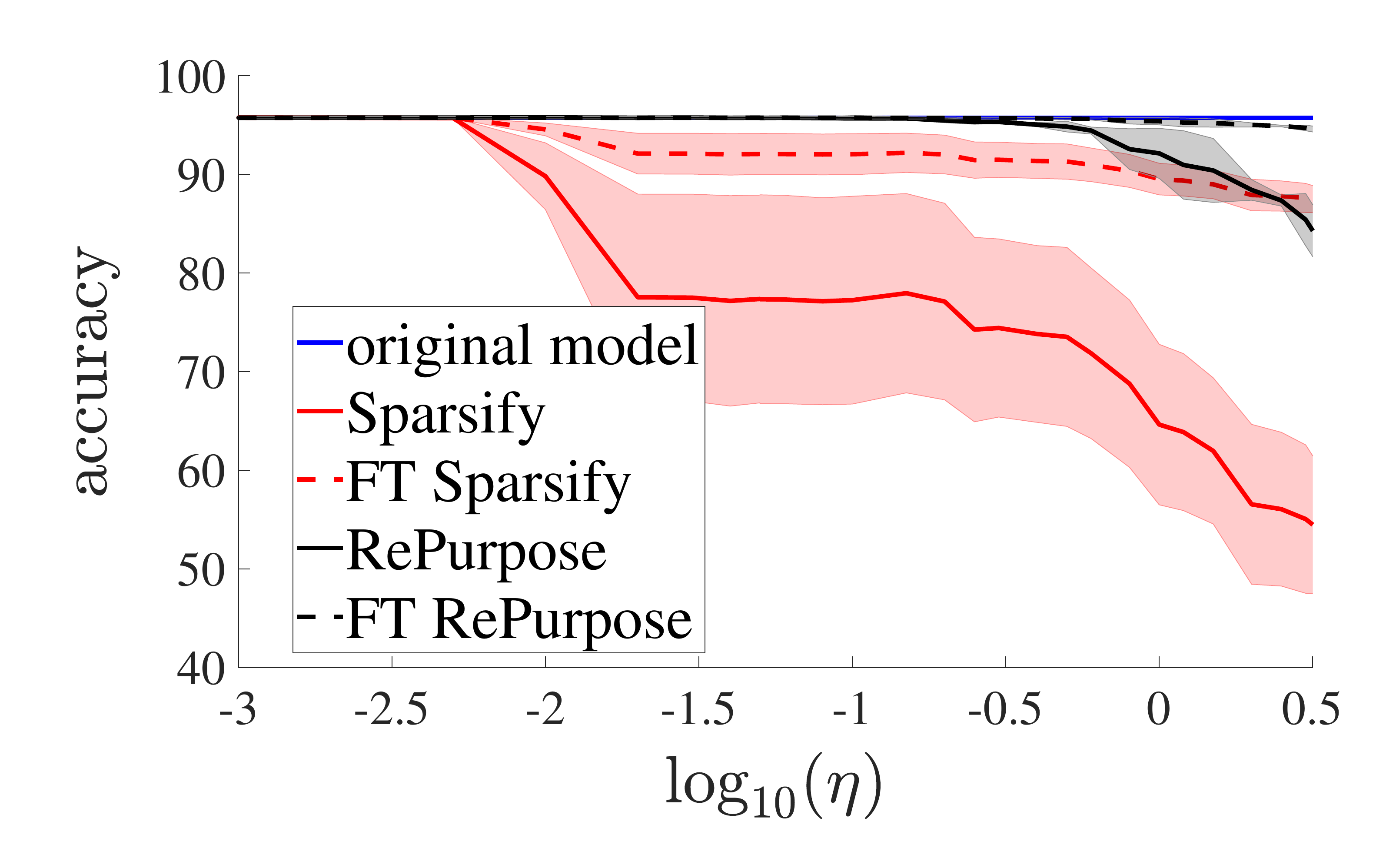}
		\caption{Accuracy vs $\eta_2$}
		\label{fig:MNIST_6Nodes-a}
	\end{subfigure}
	~
	\begin{subfigure}{0.32\linewidth}
		\centering
		\includegraphics[width=\linewidth]{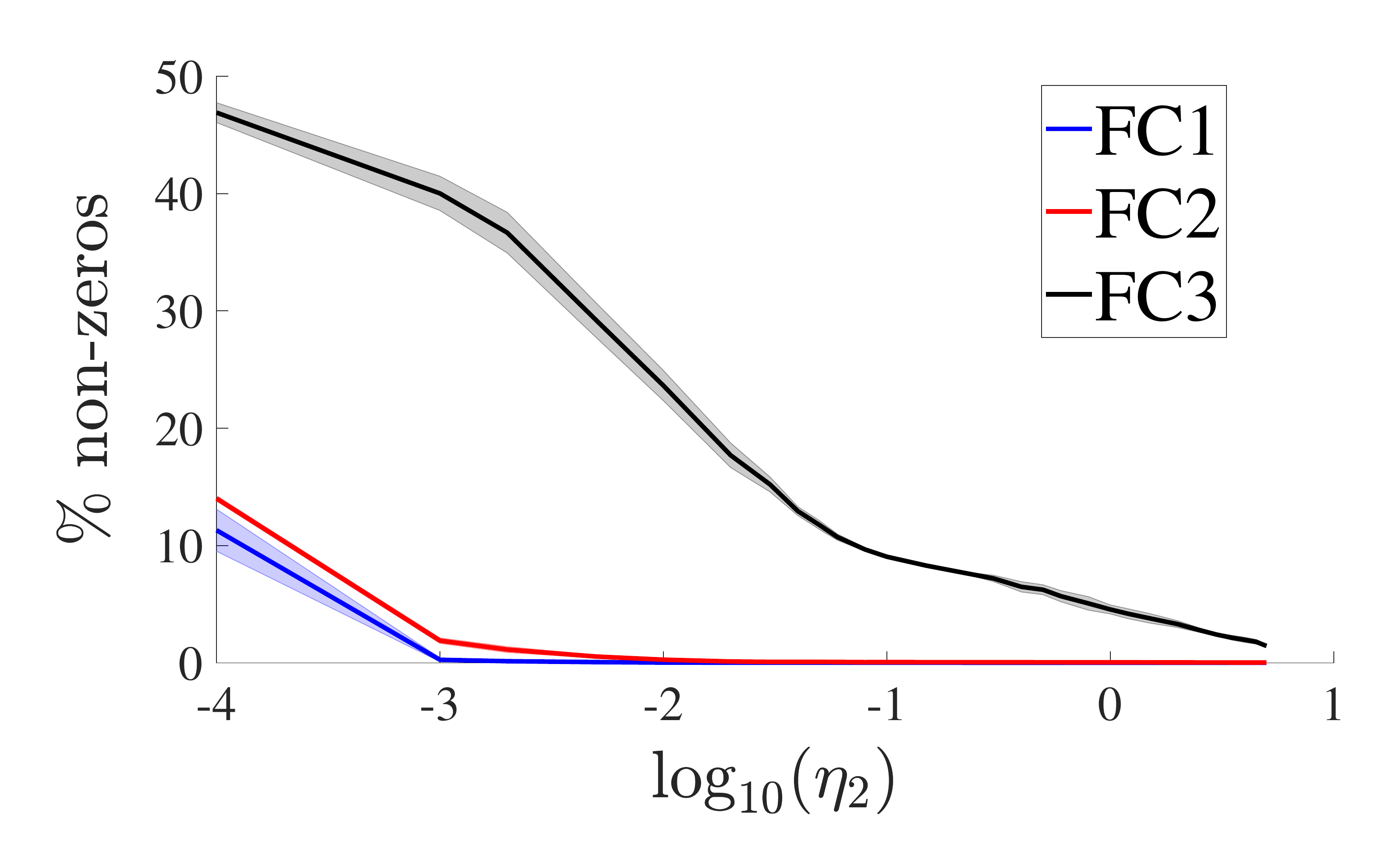}
		\caption{$\%$ of cross-edges vs $\eta_2$}
		\label{fig:MNIST_6Nodes-b}
	\end{subfigure}
	~
	\begin{subfigure}{0.32\linewidth}
		\centering
		\includegraphics[width=\linewidth]{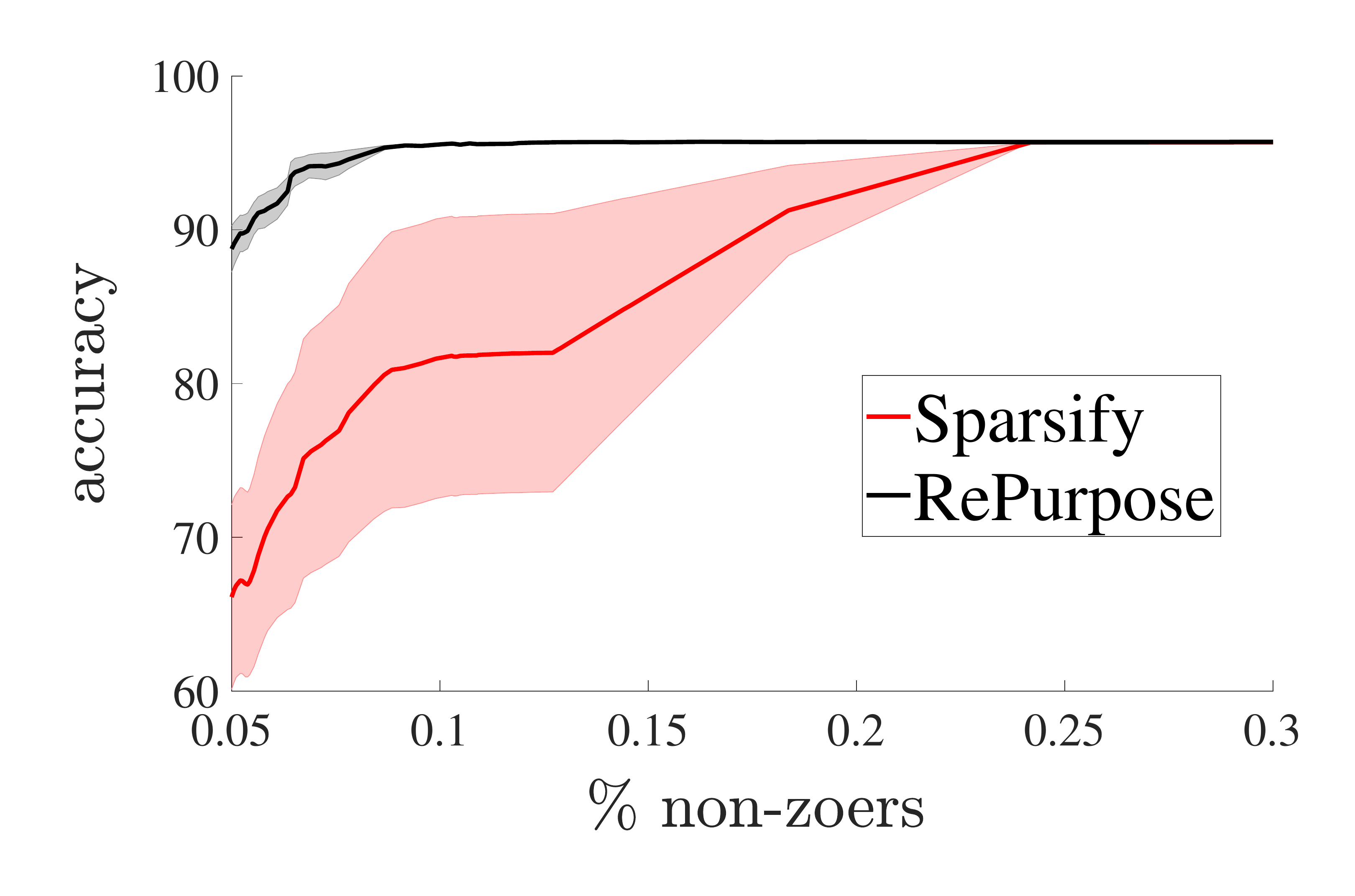}
		\caption{Accuracy vs $\%$ of cross-edges}
		\label{fig:MNIST_6Nodes-c}
	\end{subfigure}
	\vspace{-1mm}
	\caption{RePurpose vs Sparsification, a network with $6$ nodes in \textbf{Setup 2}}
	\label{fig:MNIST_6Nodes}
	\begin{subfigure}{0.31\linewidth}
		\centering
		\includegraphics[width=\linewidth]{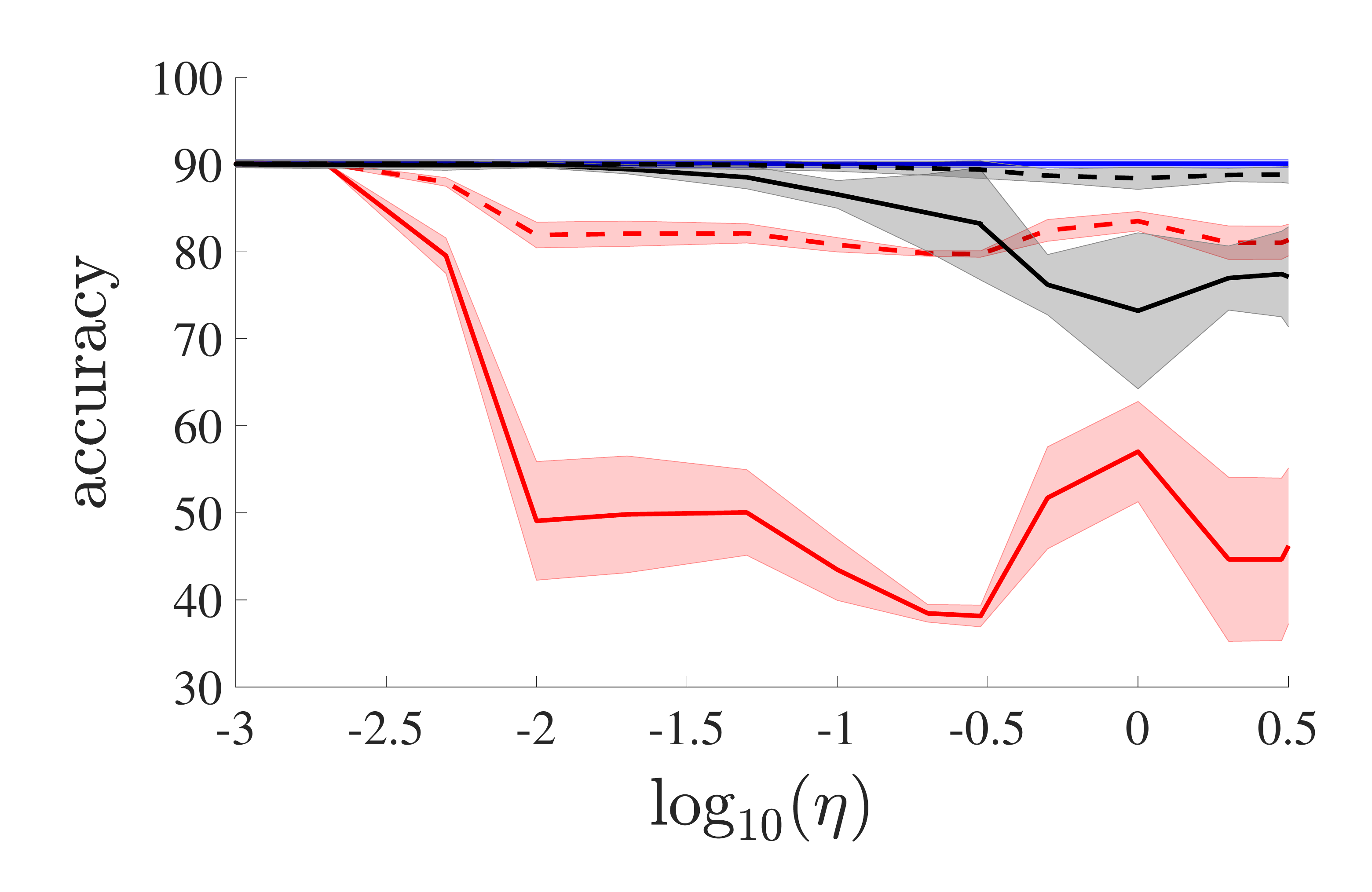}
		\caption{Accuracy vs $\eta_2$}
		\label{fig:CIFAR_2Nodes-a}
	\end{subfigure}
	~
	\begin{subfigure}{0.31\linewidth}
		\centering
		\includegraphics[width=\linewidth]{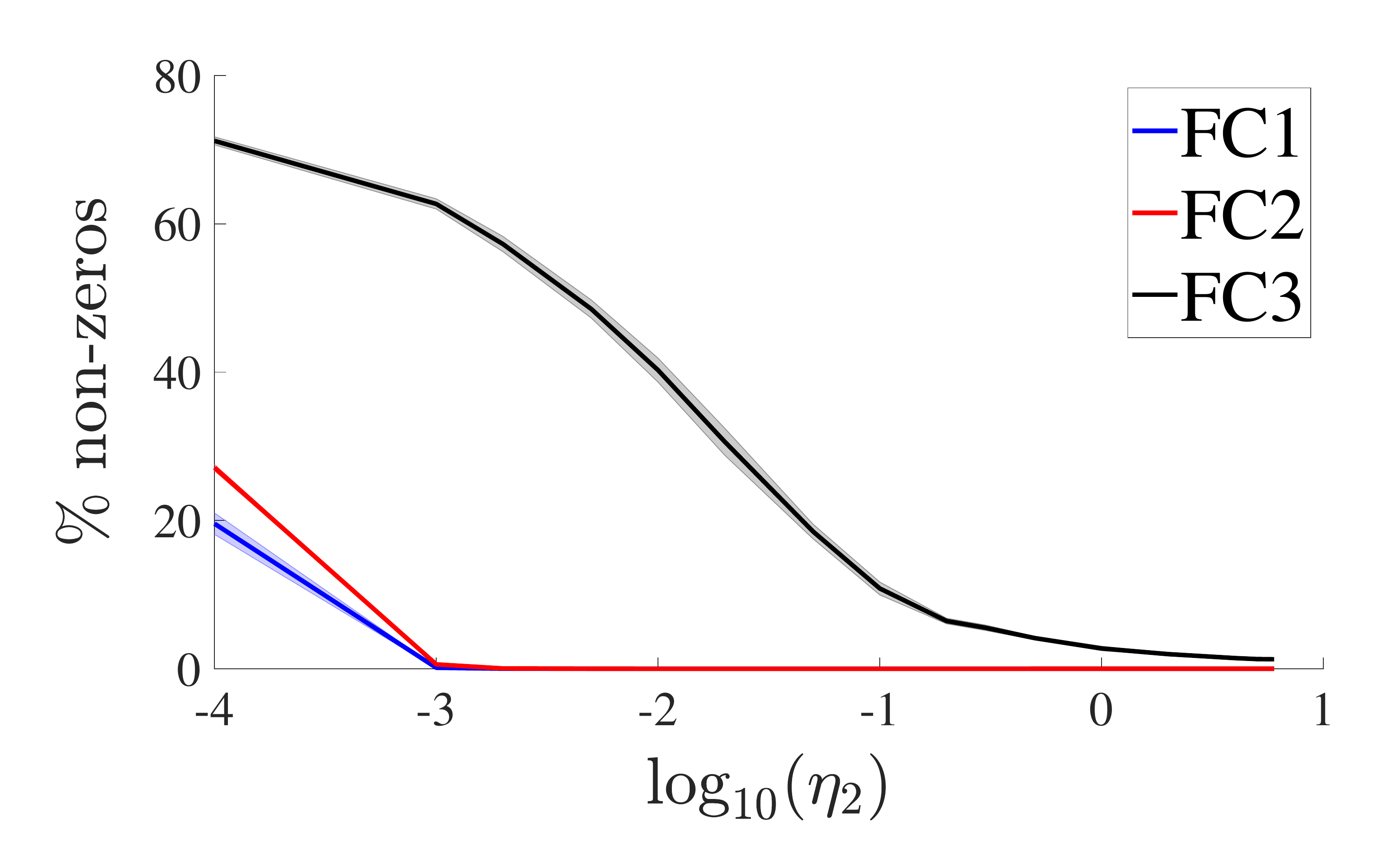}
		\caption{$\%$ of cross-edges vs $\eta_2$}
		\label{fig:CIFAR_2Nodes-b}
	\end{subfigure}
	~
	\begin{subfigure}{0.31\linewidth}
		\centering
		\includegraphics[width=\linewidth]{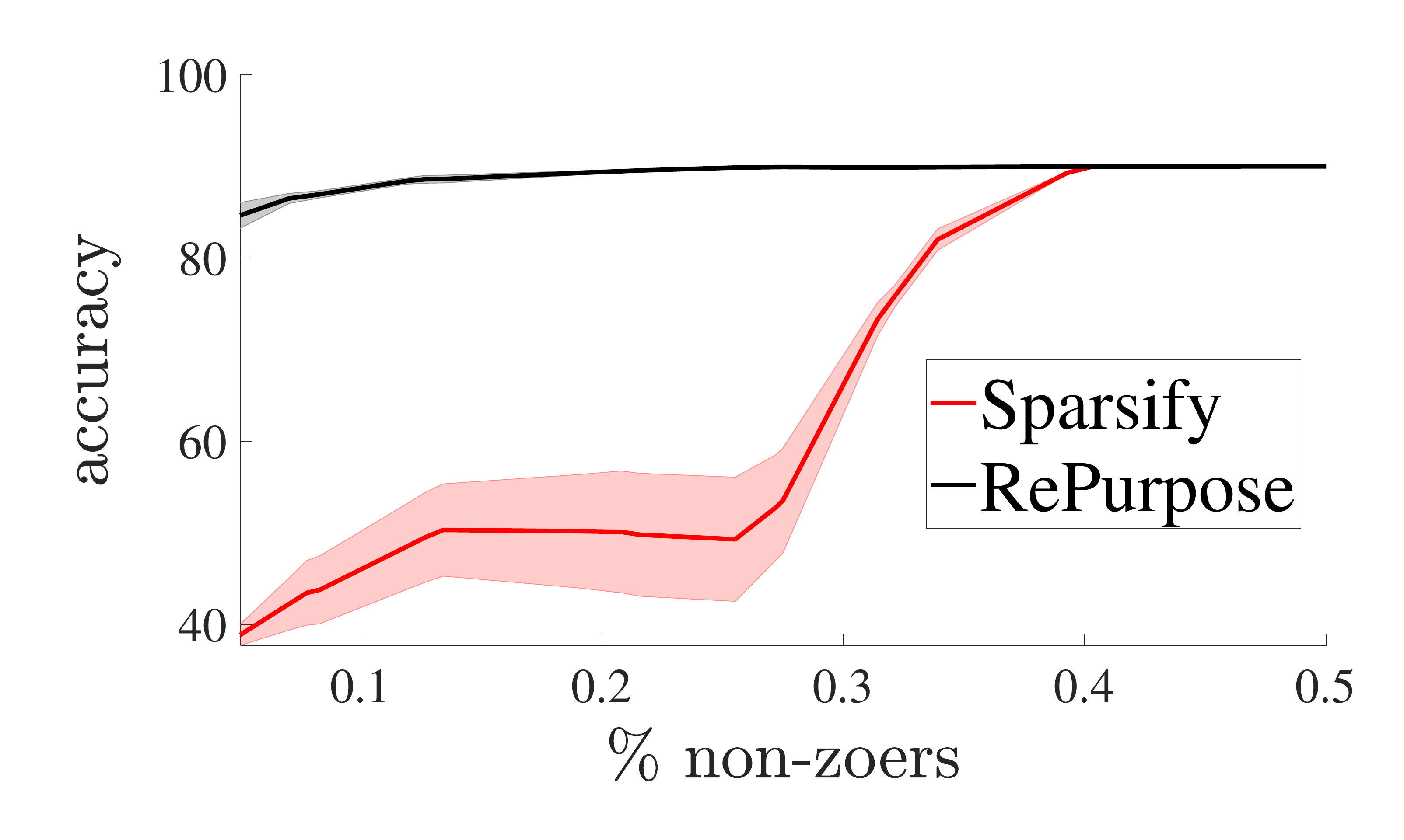}
		\caption{Accuracy vs $\%$ of cross-edges}
		\label{fig:CIFAR_2Nodes-c}
	\end{subfigure}
	\caption{RePurpose vs Sparsification, a network with $2$ nodes in \textbf{Setup 3}}
	\label{fig:CIFAR_2Nodes}
\end{figure}

\begin{wrapfigure}[8]{r}{0.57\linewidth}
	\vspace{-2mm}
	\centering
	\includegraphics[width=\linewidth]{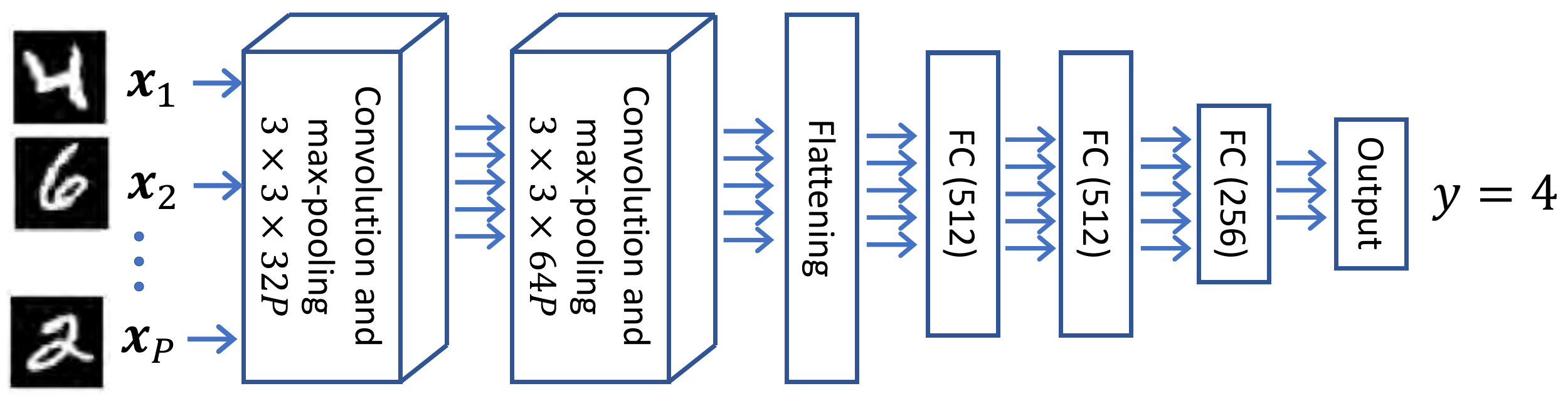}
	\vspace{-3mm}
	\caption{Structure of CNN for \textbf{Setup 2}}
	\label{fig:CNN_Setup2}
\end{wrapfigure}
\textbf{Setup 2.} Next, we consider a network of $P$ sensors where each sensor observes an image of a digit $x_i$ (from MNIST dataset) and the goal is finding the rounded average $\left[\frac{\sum_i x_i}{P}\right]$. We adapted a Lenet-5 like structure~\cite{LeCun1998Gradient} for the neural network which is trained in a central server (Fig.~\ref{fig:CNN_Setup2}), and repeated the experiments several times. Note that one might attempt to classify the digits at each individual sensor and then share the value with other nodes to compute the average. However, in addition to the increased computational complexity at each individual node, it is worth mentioning that if the accuracy of digit recognition is $\rho$, close to $1$, then the final accuracy in computing the average would be approximately $\frac{1+8\rho ^P}{9}$. For example, for a network with $6$ nodes and $\rho=0.98$, the final accuracy would be less than $90\%$. We applied the RePurpose algorithm on the trained model for distributed inference over the sensor network with different communication (cross-worker edges) constraints. Fig.~\ref{fig:MNIST_6Nodes} compares the results of RePurpose with the baseline and direct sparsification, in a network with $P=6$ sensors.

\begin{wrapfigure}[9]{r}{0.6\linewidth}
	\vspace{-2mm}
	\centering
	\includegraphics[width=\linewidth]{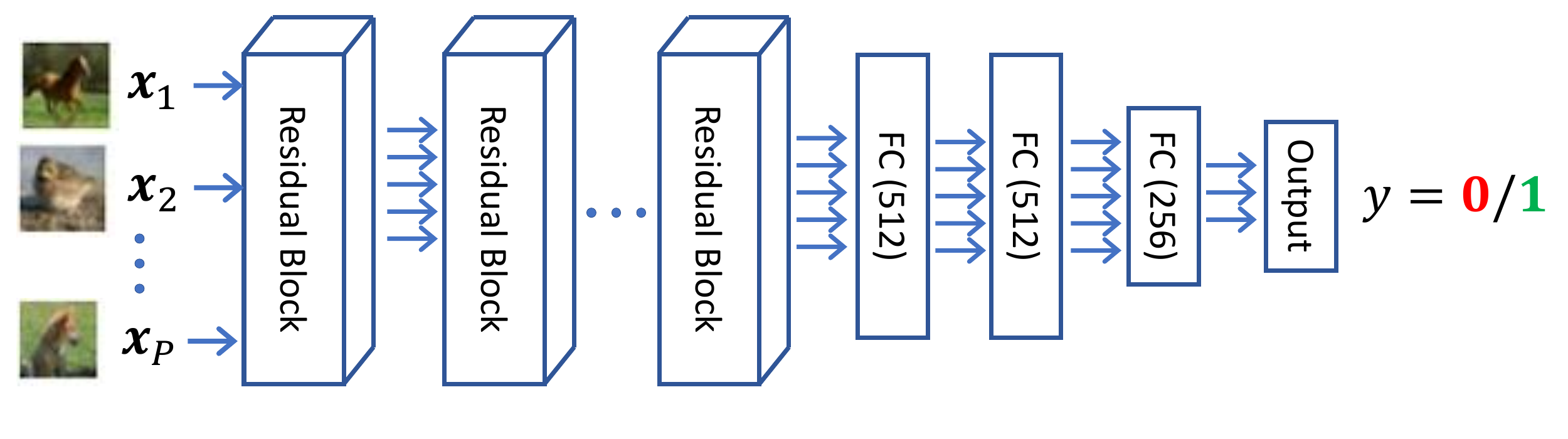}
	\caption{Structure of neural network for \textbf{Setup 3}}
	\label{fig:CNN_Setup3}
\end{wrapfigure}
\textbf{Setup 3.} Next, we consider $P$ sensors (cameras) that observe a scene and detect whether an specific object exists or not. For this purpose, we used a Resnet-like neural network~\cite{He2016Deep} over CIFAR10 and the objective is detecting the presence of a "dog" in any of the images (Fig.~\ref{fig:CNN_Setup3}). 
Fig.~\ref{fig:CIFAR_2Nodes} shows the results of RePurpose, the baseline, and direct sparsification, in a network with $P=2$ sensors.

As seen from figures \ref{fig:MNIST_6Nodes-a} and \ref{fig:CIFAR_2Nodes-a}, RePurpose significantly outperforms sparsification and although its accuracy is dropped for large $\eta_2$, with 1 or 10 epochs of post-training for MNIST and CIFAR10, respectively, ("FT RePurpose" in the figures) it achieves almost the same accuracy as the original model, while direct sparsification fails to provide good accuracy. Moreover, interestingly, RePurpose sparsifies the cross-edges between workers significantly for the hidden layers. The restructured model can achieve the same performance as the original model by using less than $0.0003$ of the cross-edges (i.e., between $10$ to $30$ connections out of more than $100000$ edges between workers). Finally, figures \ref{fig:MNIST_6Nodes-c} and \ref{fig:CIFAR_2Nodes-c} compare the accuracy vs the cross-communication between workers. Clearly, direct sparsification performs well \emph{only} when there are enough number of cross-edges between the workers, while the accuracy of the model obtained by RePurpose does not change for a vast sparsity range.

Finally, it is worth mentioning that in the naive approach to inference over the sensor network, each node has to transmit its observations to other nodes, hence the communication between any two pair of nodes would be $784$ or $1024$ values for \textbf{Setups 2} and $\textbf{3}$, respectively. However, RePurpose can achieve the same accuracy with less than \emph{200 total communicated values across the entire network}.

\subsection{System Evaluations}
\label{sec:eval}

\begin{table}[!t]
\caption{
Target Accelerator Evaluation Platforms 
}
\centering
\begin{tabular}{ccccc} 
\toprule
\textbf{Name} & \textbf{Node Compute} & \textbf{Node Memory} & \textbf{Network Bandwidth} & \textbf{Number of Nodes} \\
\hline
\textbf{Datacenter}  & 125 TOPS & 4GB & 150 GB/s (NVLink) & 1-32\\
\hline
\textbf{Edge}  & 0.5 TOPS & 1GB & 100 MB/s (Ethernet) & 1-32\\
\bottomrule 
\label{table:system_config}
\end{tabular}
\vspace{-5mm}
\end{table}

\textbf{Methodology-} 
We evaluate RePurpose on two distributed accelerator platforms,
described in Tbl.~\ref{table:system_config}, 
simulated using ASTRA-sim~\cite{astrasim}.
ASTRA-sim is an open-source 
distributed Deep Learning platform 
simulator that models cycle-level 
communication behavior in details for any partitioning strategy across multiple interconnected accelerator nodes.
ASTRA-sim takes the compute cycles for each 
layer of the model as an external input, and manages  communication scheduling
similar to communication libraries like 
NVIDIA NCCL~\cite{nccl}.
We obtained compute cycles for the Datacenter configuration from 
a NVIDIA V100 GPU implementation, and for the Edge 
configuration (e.g., sensor network) 
from a separate DNN accelerator simulator~\cite{scalesim}.

We tried to stress the aforementioned platforms under various sized problems to show the efficiency of RePurpose. 
In all models, we assumed a stack of 5 layers with same number of neurons. In our notation, $N$ refers to the number of neurons per layer (or matrix dimensions).
For the datacenter system, $N$ varies from $1K$ to  
$1M$, while for edge system the variation is from $1K$ to $32K$. We also assumed strict ordering between current communication and computation of next layer, meaning that each node begins computation of each layer only when it has all inputs available.

We picked 4 different flavors of RePurpose with 50\%, 75\%, 90\% and 99\% sparsity factor named as RP-50, RP-75, RP-90, and RP-99, respectively. In addition, we changed the number of worker nodes from 1 to 32 for both system configurations.

\begin{wrapfigure}[12]{r}{0.51\textwidth}
	\centering
	\includegraphics[width=\linewidth]{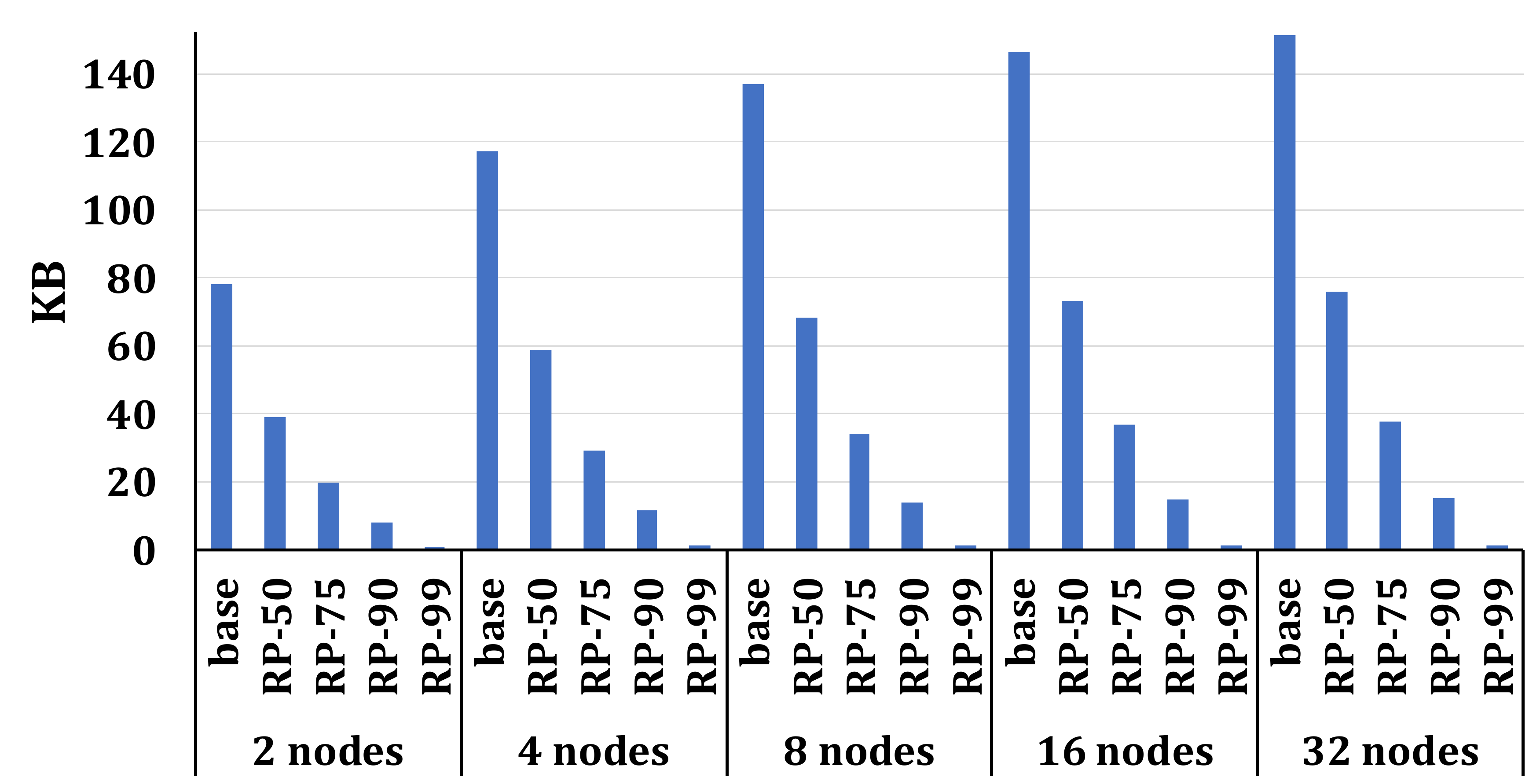}
	\caption{Theoretical amount of data each node needs to send out for $N=8K$.}
	\label{fig:theoriticalcomms}
\end{wrapfigure}
\textbf{Results-} 
Fig.~\ref{fig:theoriticalcomms} shows the total amount of data that each node needs to send out for one input sample for $N=8K$. Clearly, specification has the linear effect on the amount of communicating data. On the other hand, partitioning across more nodes also increases the total communicating data. But the increase in rate diminishes as the number of nodes increases, converges to 2X more data compared to the case of 2 nodes.  

To further investigate the effect of RePurpose in reducing the computation and communication times, Fig.~\ref{fig:detailedResults} shows the simulation results of the communication and computation breakdown for the baseline system and RePurpose for $N=8k$. 
As seen from Fig.~\ref{fig:datacenter_8k}, in a datacenter system, on average and across different number of nodes, RP-50, RP-75, RP-90 and RP-99 achieve 1.7$\times$, 2.76$\times$, 4.77$\times$ and 10.47$\times$ speed-up in computations, respectively. The average improvement for communication ratio is 1.2$\times$, 1.45$\times$, 1.74$\times$ and 1.75$\times$, respectively. The reason for lower improvements of communication time is that due to NVLink's high bandwidth. For $N=8K$, network communication time is mostly network latency limited. Hence, reduction in input size does not correspond to linear reduction in communication time.

Fig.~\ref{fig:edge_8k} shows the similar results but for edge system. Here, due to much lower network bandwidth, the effect of communication is more considerable. On average applying RP-50, RP-75, RP-90 and RP-99 improve computation times by 1.7$\times$, 2.77$\times$, 4.78$\times$ and  11.01$\times$, respectively. This value for communication is 1.2$\times$, 1.38$\times$, 1.82$\times$ and 3.04$\times$ respectively. As the number of nodes grow, the communication gap between the baseline and RePurpose decreases. This is mostly because of the congestion in the network (e.g. switch) that decreases the effect of benefits gained by RePurpose.

\begin{figure*}            
	\begin{subfigure}{0.5\linewidth}
	\setlength{\abovecaptionskip}{0pt}
    \setlength{\belowcaptionskip}{0pt}
		\centering
		\includegraphics[width=.9\linewidth]{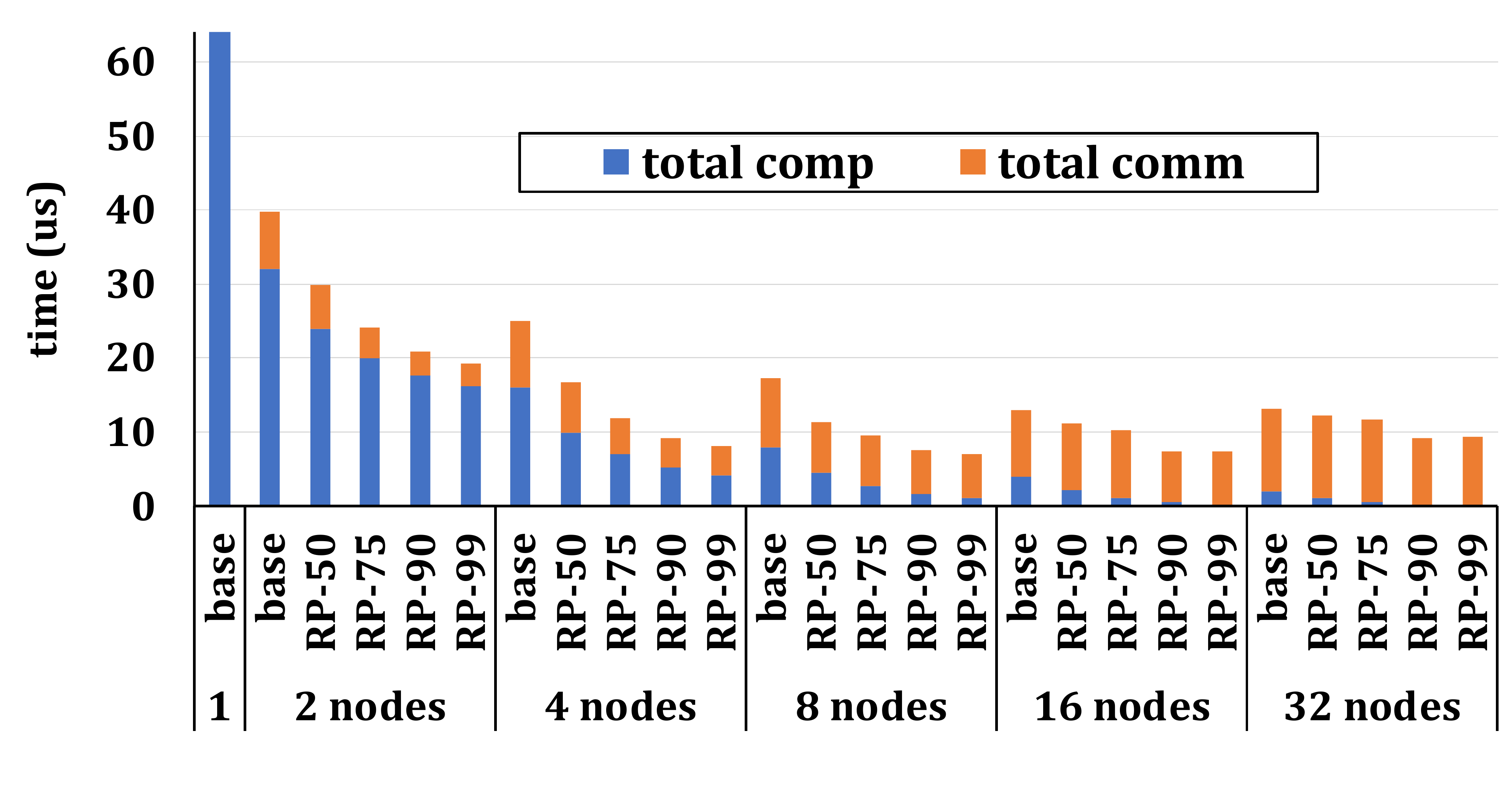}
		\caption{Datacenter Platform results}
		\label{fig:datacenter_8k}
	\end{subfigure}
	\begin{subfigure}{0.5\linewidth}
		\setlength{\abovecaptionskip}{0pt}
    \setlength{\belowcaptionskip}{0pt}
		\centering
		\includegraphics[width=.9\linewidth]{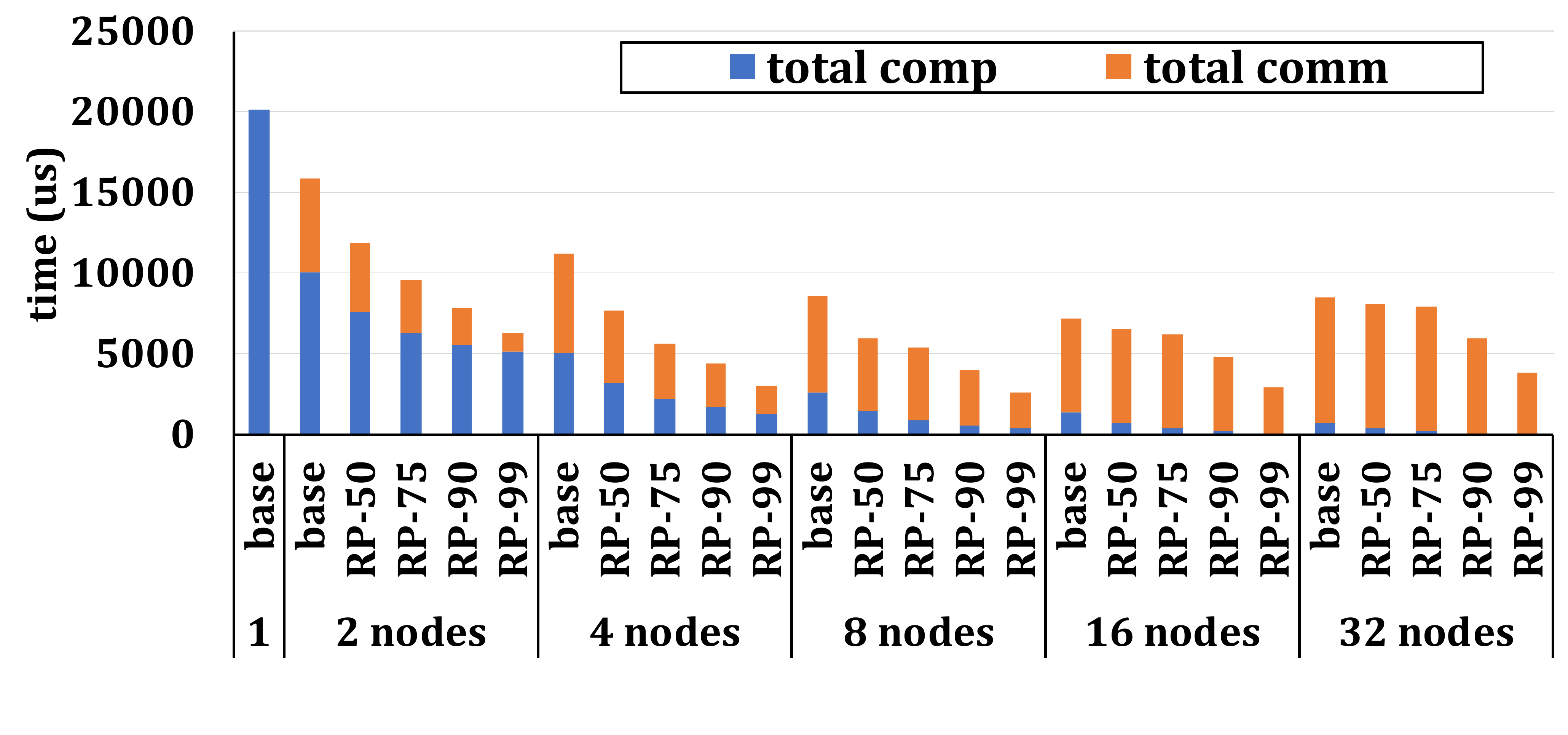}
		\caption{Edge Platform results}
		\label{fig:edge_8k}
	\end{subfigure}
	\caption{Communication and computation breakdown across different systems and $N=8K$}
	\label{fig:detailedResults}
\vspace{3mm}
	\begin{subfigure}{0.5\linewidth}
	\setlength{\abovecaptionskip}{-3pt}
    \setlength{\belowcaptionskip}{0pt}
		\centering
		\includegraphics[width=0.9\linewidth]{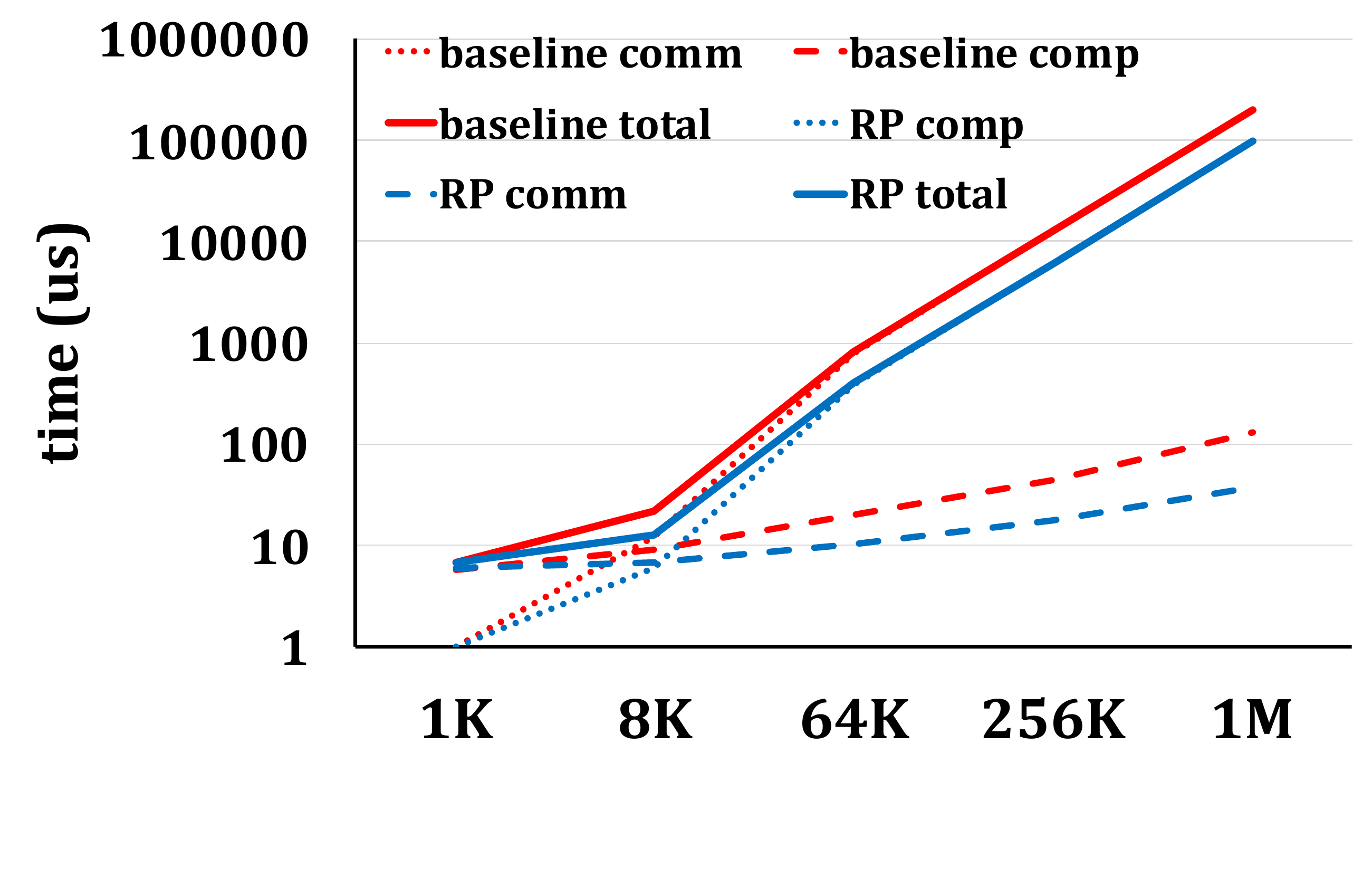}
		\vspace{-4mm}
		\caption{Datacenter Platform results}
		\label{fig:datacenterTrend}
	\end{subfigure}
	\begin{subfigure}{0.5\linewidth}
		\setlength{\abovecaptionskip}{-3pt}
    \setlength{\belowcaptionskip}{0pt}
		\centering
		\includegraphics[width=0.9\linewidth]{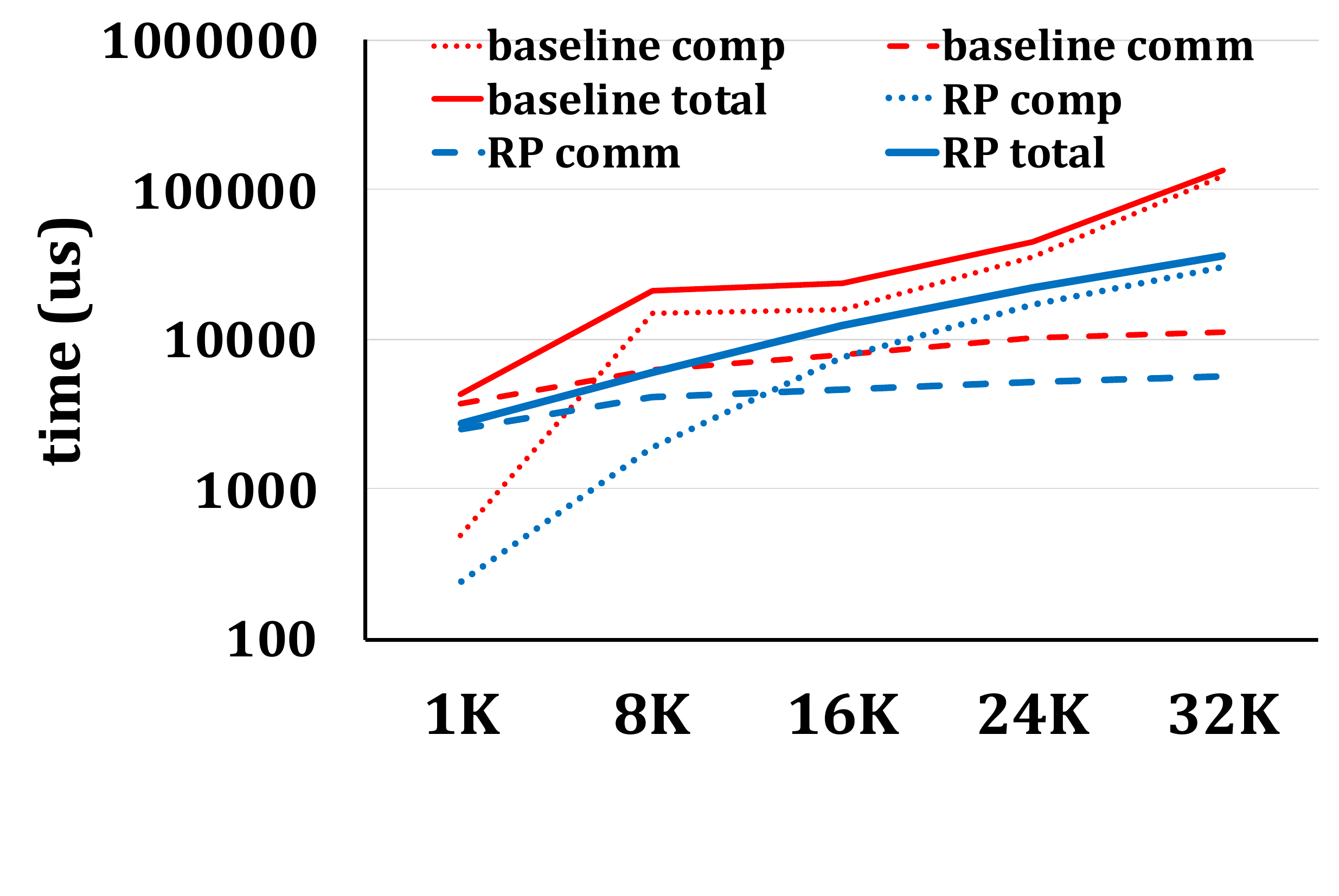}
		\vspace{-4mm}
		\caption{Edge Platform results}
		\label{fig:edgeTrend}
	\end{subfigure}
	\caption{The effect of communication vs. computation times as the model size $N$ grows}
	\label{fig:trend}
	\vspace{-5mm}
\end{figure*}

Fig.~\ref{fig:trend} shows how communication, computation and total times change as the the number of neurons grows. For each network size, computation and communication times are averaged across different sparsity factors and node counts. For datacenter system (Fig\ref{fig:datacenterTrend}), computation is the dominant factor. This is expected since the computation grows as $O(N^2)$ while communication increases as $O(N)$. Since the network band-width is very high in datacenter, the effect of communication is negligible. In general, the total time ratio increase from 1.01$\times$ in $N=1K$ to 2.06$\times$ in $N=1M$.
On the other hand, communication remains a considerable factor in the edge systems (Fig.~\ref{fig:edgeTrend}) due to: ($i$) low network bandwidth, and ($ii$) lower dimensions of workloads on edge systems. The total time improvement for edge system is $1.55\times$ for $N=1K$ and it increases to $3.8\times$ for $N=32K$.


\section{Conclusion}
\label{sec:conclusion}
In this paper, we considered the problem of efficient parallel distributed inference of an already trained deep model over a cluster of processing units or a sensor network. Required communication and synchronization among processing units or network nodes (i.e., workers) can adversely affect the computation time. Moreover, in the wireless sensor networks, it may significantly increase the power consumption due to the transmission of large amount of data.
We claimed that traditional approaches to prune or compress the deep models fail to consider the constraints imposed in such distributed inference systems. 
To overcome the shortcomings of the existing methods, we devised \emph{RePurpose}, a framework to restructure the deep model by rearranging the neurons, optimum assignment of neurons to the workers, and then pruning the parameters, such that the dependency among workers is reduced. 
We showed that RePurpose can significantly reduce the number of cross-communication between workers and improve the computation time significantly, while the performance loss of the modified model is remained negligible.



\bibliographystyle{plainnat}
\bibliography{references}

\newpage
\appendix
\section{Complexity of Naive Direct Partitioning}
Consider distributing processing of a layer of a deep neural network with $N$ neurons over $P$ workers. Without assuming any constraint on the number of neurons per worker, there are $P$ possible assignments for each neuron, hence, the total possible neuron assignments to the workers would be $P^N$. 

Now, assume that exactly $n_k$ neurons have to be assigned to the $k$-th worker, where $\sum_k n_k = N$. Clearly, there are 
\begin{equation*}
	{N \choose {n_1,n_2,\ldots,n_P}}
\end{equation*}
possible neuron assignment to the workers. 
To have a relatively balanced neuron assignment (i.e., no worker or a small subset of workers has to process almost all signals), we assume that $n_k=c_k\,N$, where $c_k =\Theta(1/P)$, i.e., there exists $\alpha,\beta>0$ such that $\alpha N/P \leq n_k \leq \beta N/P$. Using Stirling's approximation for factorial, $n_k!\sim \sqrt{2\pi n_k}\, (\frac{n_k}{e})^{n_k}$, and noting that $n_k=N\Theta(\frac{1}{P})$, $\sum_k n_k = N$, we have
\begin{align*}
	{N \choose {n_1,n_2,\ldots,n_P}} &\sim \frac{\sqrt{2\pi N}\,(\frac{N}{e})^N}{\prod_{k=1}^{P}\sqrt{2\pi n_k}\, (\frac{n_k}{e})^{n_k}} \\ 
	&= \frac{\sqrt{2\pi N}\,(\frac{N}{e})^N}{\prod_{k=1}^{P}\sqrt{2\pi N\Theta(\frac{1}{P})}\, \big(\frac{N\Theta(\frac{1}{P})}{e}\big)^{n_k}} \\
	&= \frac{1}{(2\pi N)^{\frac{P}{2}-1}}\,\frac{1}{\Theta(\frac{1}{P^{N+0.5}})} \\
	&= \Theta \big(P^{N+0.5}\, N^{1-\frac{P}{2}}\big).
\end{align*}
Therefore, the direct approach to find good neuron assignment for parallel distributed inference requires evaluation of $\BigO(P^N)$ different assignments, which for large number of neurons or number of workers becomes prohibitive.
 
\section{Application of RePurpose in Deep Neural Networks}
Recall that at the core of the RePurpose algorithm is solving the optimization problem and finding the cost of assigning neuron $i$ to worker $j$, given by
\begin{equation}
	C_{ji} = \min_{\whvect{w}_i} \|\vect{w}_i-\whvect{w}_i\|_2^2 + \eta_1 \|\whvect{w}_i\|_0 + \eta_2 \|\whvect{w}_{i,\backslash j}\|_0,
	\label{eqn:PRP_param_cost2}
\end{equation}
where $\eta_1$ and $\eta_2$ control the trade-off between the error, sparsity, and cross-communication.

The basic RePurpose function and its application to a deep neural network with weights and biases $\{\vect{W}^{(l)},\vect{b}^{(l)}\}$ are summarized in Algorithms~\ref{alg:RePurposeP} and \ref{alg:RePurpose_DNN}, respectively. In Alg.~\ref{alg:RePurpose_DNN}, $n_k^{(l)}$ is the number of neurons in layer $l$ being assigned to worker $k$, and $\mathcal{H}_{\vect{E}}(\cdot)$ is the (modified) element-wise hard-thresholding operator, defined as
\begin{equation}
	\vect{Y}=\mathcal{H}_{\vect{E}}(\vect{X}):\quad 
	Y_{ij} = \left\{
		\begin{array}{ll}
			0 & \textup{if } |X_{ij}|^2\leq E_{ij} \\
			X_{ij} & \textup{o.w.}
		\end{array}
	\right.
\end{equation}

%

\centerline{
\begin{minipage}{0.96\linewidth}
\begin{algorithm}[H]
	\SetAlgoLined
	\caption{\small Applying RePurpose to Deep Neural Networks}
	\label{alg:RePurpose_DNN}
	\DontPrintSemicolon
	\KwIn{$\{\vect{W}^{(l)}\}_{_l}$, $\{\vect{b}^{(l)}\}_{_l}$, $\{n_k^{(l)}\}_{_{k,l}}$, $\eta_1$, $\eta_2$}
	\KwOut{$\{\vect{\Pi}^{(l)}\}_l$, $\{\whvect{W}^{(l)}\}_l$, $\{\whvect{b}^{(l)}\}_l$}
	$\vect{E}=\eta_1+\eta_2\vect{M}$\;
	$\vect{\Pi}^{(0)}\leftarrow \vect{I}$\;
	\For{layers $l=1,\ldots,L$}
	{
		$\vect{T}\leftarrow\vect{\Pi}^{(l-1)} \vect{W}^{(l)}$\label{alg:RePurpose_DNN_applyPi}\;
		$\vect{\Pi}^{(l)} \leftarrow \Function{RePurpose}(\vect{T},\{n_k^{(l)}\}_{_k},\eta_1,\eta_2)$\;
		$\whvect{W}^{(l)} \leftarrow \mathcal{H}_{\vect{E}} \big(\vect{T}(\vect{\Pi}^{(l)})^\Tr\big)$\;
		$\whvect{b}^{(l)} \leftarrow \vect{\Pi}^{(l)} \vect{b}^{(l)}$\;
	}
\end{algorithm}
\end{minipage}
}

Recall that when applying RePurpose to layers of a neural network, permuting neurons of layer $l$ with matrix $\vect{\Pi}$ changes the signal of that layer by $\whvect{y}^{(l)}=\vect{\Pi y}^{(l)}$ and affects the weight matrix of that layer by $\vect{W}^{(l)}\vect{\Pi}^\Tr$. As a result, $\whvect{x}^{(l+1)}=\vect{\Pi x}^{(l+1)}$ and to have the same signal at the next layer, $\vect{y}^{(l+1)}$, the weight matrix of layer $l+1$ should be modified as $\vect{\Pi W}^{(l+1)}$. Line~\ref{alg:RePurpose_DNN_applyPi} of Alg.~\ref{alg:RePurpose_DNN} accounts for these adjustments. 

\section{Performance Guarantee of RePurpose}
Consider an arbitrary neural network with $L$ layers and parameters $\{\vect{\theta}^{(l)}\}_{l=1}^L$, where $\vect{\theta}^{(l)}=\{\vect{W}^{(l)},\vect{b}^{(l)}\}$ is the parameters of the $l$-th layer. Let $\vect{x}^{(l)}$ be the input signal to the $l$-th layer. Then, the output of the layer (input to the next layer) would be given by
\begin{equation}
\vect{y}^{(l)} = (\vect{W}^{(l)})^\Tr\vect{x}^{(l)}+\vect{b}^{(l)},\quad\quad\vect{x}^{(l+1)}=\sigma(\vect{y}^{(l)}),
\end{equation}
where $\sigma(\cdot)$ is the activation function.

To analyze the performance of the modified neural network, assume that the original neural network has the following properties:
\begin{itemize}[leftmargin=25pt,itemsep=1pt,partopsep=1pt,parsep=1pt,topsep=1pt]
	\item [\textbf{A1.}] The activation functions are $1$-Lipschitz, i.e., for all $u, v$, $|\sigma(u)-\sigma(v)|\leq |u-v|$.
	\item [\textbf{A2.}] The Frobenius norms of the weights of the neural network are bounded, i.e., for some constant $\tau>0$, $\|\vect{W}^{(l)}\|_F\leq \tau$, for all layers $l=1,\ldots, L$. 
	\item [\textbf{A3.}] The signals in the neural networks are bounded, i.e., there exists a constant $B>0$ such that for input signal $\vect{x}^{(1)}=\vect{x}_{in}$, and forward signals $\{\vect{x}^{(l)}\}_{l=2}^{L}$ (outputs of the hidden layers), $\|\vect{x}^{(l)}\|_2\leq B$ for $l=1,\ldots,L$.
\end{itemize}

Moreover, suppose that the parameters $\eta_1$ and $\eta_2$ at each call of the \Function{RePurpose} are adjusted such that the solution of Lagrangian formulation \eqref{eqn:Lagrangian_RePurposeP_OPT}, given by \Function{RePurpose}, is also the solution of the following constrained optimization problem
\begin{equation}
	\min_{\whvect{W},\vect{\Pi}} \|\vect{M}\odot\whvect{W}\|_0\quad\quad \subjto~\|\whvect{W}-\vect{T\Pi}^\Tr\|_F^2\leq \epsilon.
	\label{eqn:RePurposePP-OPT}
\end{equation}
Hence, by Alg.~\ref{alg:RePurpose_DNN} and the cascade application of RePurpose, the modified weight matrix of the $l$-th layer of neural network satisfies $\|\whvect{W}^{(l)}-\vect{\Pi}^{(l-1)}\vect{W}^{(l)}(\vect{\Pi}^{(l)})^\Tr\|_F^2\leq \epsilon$. For the simplicity in notations, let  $\varepsilon=\sqrt{\epsilon}$.

\begin{theorem}
	For an input data $\vect{x}$, let $\vect{y}$ and $\whvect{y}$ be the outputs of the original and RePurposed neural network, respectively. If $\vect{\Pi}$ is the permutation of the final output neurons in the RePurposed neural network, then under assumptions \textbf{A1-3},
	\begin{equation}
		\|\whvect{y}-\vect{\Pi y}\|_2 \leq \varepsilon \frac{(\tau+\varepsilon)^L-1}{\tau+\varepsilon-1} B.
	\end{equation}
	Especially, if the parameters of the neural network are normalized such that $\|\vect{W}^{(l)}\|_F=1$, then $\|\whvect{y}-\vect{\Pi y}\|_2 \leq \big((1+\varepsilon)^L-1\big) B.$
\end{theorem}
\begin{proof}
	Let $\vect{x}^{(l)}$ and $\whvect{x}^{(l)}$ be the signals in the original and modified neural network, corresponding to the input $\vect{x}$.
	Note that $\vect{\Pi}^{(0)}=\vect{I}$ and the input to both networks are the same, $\vect{x}^{(1)}=\whvect{x}^{(1)}=\vect{x}$. Let $\vect{\Pi}^{(l)}$ and $\{\whvect{W}^{(l)}, \whvect{b}^{(l)}\}$ be the permutation matrix and parameters of the modified neural network, found via \eqref{eqn:RePurposeP-OPT}. Therefore, using $\vect{x}^{(l+1)}=\sigma((\vect{W}^{(l)})^\Tr\vect{x}^{(l)}+\vect{b}^{(l)})$, for any arbitrary layer $l$,
	\begin{align*}
		&\|\vect{\Pi}^{(l)}\vect{x}^{(l+1)}-\whvect{x}^{(l+1)}\|_2 
		\\ 
		=~
		&\|\vect{\Pi}^{(l)}\sigma((\vect{W}^{(l)})^\Tr\vect{x}^{(l)}+\vect{b}^{(l)})-\sigma((\whvect{W}^{(l)})^\Tr\whvect{x}^{(l)}+\whvect{b}^{(l)})\|_2 
		\\
		\stackrel{(a)}{=}~
		&\|\sigma(\vect{\Pi}^{(l)}(\vect{W}^{(l)})^\Tr\vect{x}^{(l)}+\vect{\Pi}^{(l)}\vect{b}^{(l)})-\sigma((\whvect{W}^{(l)})^\Tr\whvect{x}^{(l)}+\whvect{b}^{(l)})\|_2 
		\\
		\stackrel{(b)}{\leq}~
		&\|(\vect{\Pi}^{(l)}(\vect{W}^{(l)})^\Tr\vect{x}^{(l)}+\vect{\Pi}^{(l)}\vect{b}^{(l)})-((\whvect{W}^{(l)})^\Tr\whvect{x}^{(l)}+\whvect{b}^{(l)})\|_2 
		\\
		\stackrel{(c)}{=}~
		&\|\vect{\Pi}^{(l)}(\vect{W}^{(l)})^\Tr\vect{x}^{(l)}-(\whvect{W}^{(l)})^\Tr\whvect{x}^{(l)}\|_2 
		\\
		= ~
		&\|\big(\vect{\Pi}^{(l-1)}\vect{W}^{(l)}(\vect{\Pi}^{(l)})^\Tr-\whvect{W}^{(l)}\big)^\Tr\whvect{x}^{(l)}+
		\big(\vect{W}^{(l)}(\vect{\Pi}^{(l)})^\Tr\big)^\Tr\big((\vect{\Pi}^{(l-1)})^\Tr\whvect{x}^{(l)}-\vect{x}^{(l)}\big)\|_2 
		\\
		\leq ~ 
		&\|\big(\vect{\Pi}^{(l-1)}\vect{W}^{(l)}(\vect{\Pi}^{(l)})^\Tr-\whvect{W}^{(l)}\big)^\Tr\whvect{x}^{(l)}\|_2+
		\|\big(\vect{W}^{(l)}(\vect{\Pi}^{(l)})^\Tr\big)^\Tr\big((\vect{\Pi}^{(l-1)})^\Tr\whvect{x}^{(l)}-\vect{x}^{(l)}\big)\|_2 
		\\
		\stackrel{(d)}{\leq} ~
		&\|\vect{\Pi}^{(l-1)}\vect{W}^{(l)}(\vect{\Pi}^{(l)})^\Tr-\whvect{W}^{(l)}\|_F \, \|\whvect{x}^{(l)}\|_2 + 
		\|\vect{W}^{(l)}(\vect{\Pi}^{(l)})^\Tr\|_F \, \|(\vect{\Pi}^{(l-1)})^\Tr\whvect{x}^{(l)}-\vect{x}^{(l)}\|_2 
		\\
		= ~
		&\|\whvect{W}^{(l)}-\vect{\Pi}^{(l-1)}\vect{W}^{(l)}(\vect{\Pi}^{(l)})^\Tr\|_F\,\|\whvect{x}^{(l)}\|_2 + 
		\|\vect{W}^{(l)}\|_F\,\|\whvect{x}^{(l)}-\vect{\Pi}^{(l-1)}\vect{x}^{(l)}\|_2
		\\
		\stackrel{(e)}{\leq} ~
		&\varepsilon\|\whvect{x}^{(l)}\|_2 +  \tau\|\vect{\Pi}^{(l-1)}\vect{x}^{(l)}-\whvect{x}^{(l)}\|_2 
		\\
		\leq ~
		&\varepsilon\big(\|\whvect{x}^{(l)}-\vect{\Pi}^{(l-1)}\vect{x}^{(l)}\|_2+\|\vect{\Pi}^{(l-1)}\vect{x}^{(l)}\|_2\big) +  \tau\|\vect{\Pi}^{(l-1)}\vect{x}^{(l)}-\whvect{x}^{(l)}\|_2 
		\\
		= ~
		&(\tau+\varepsilon)\|\vect{\Pi}^{(l-1)}\vect{x}^{(l)}-\whvect{x}^{(l)}\|_2 + \varepsilon\|\vect{x}^{(l)}\|_2
		\\
		\leq ~
		&(\tau+\varepsilon)\|\vect{\Pi}^{(l-1)}\vect{x}^{(l)}-\whvect{x}^{(l)}\|_2 + \varepsilon B		
	\end{align*}
	where $(a)$ is because $\vect{\Pi}\sigma(\vect{z})=\sigma(\vect{\Pi z})$ for arbitrary permutation $\vect{\Pi}$ and vector $\vect{z}$, $(b)$ is because $\sigma(\cdot)$ is $1$-Lipschitz, $(c)$ is due to the fact that $\whvect{b}^{(l)}=\vect{\Pi}^{(l)}\vect{b}^{(l)}$, $(d)$ is from $\|\vect{Az}\|_2\leq \|\vect{A}\|_2\|\vect{z}\|_2\leq \|\vect{A}\|_F\|\vect{z}\|_2$ for arbitrary $\vect{A}$ and $\vect{z}$, and $(e)$ is by assumption \textbf{A2} and \eqref{eqn:RePurposeP-OPT}. Therefore,
	\begin{equation}
		\|\vect{\Pi}^{(l)}\vect{x}^{(l+1)}-\whvect{x}^{(l+1)}\|_2 \leq (\tau+\epsilon)\|\vect{\Pi}^{(l-1)}\vect{x}^{(l)}-\whvect{x}^{(l)}\|_2 + \varepsilon B.
		\label{eqn:error_recursion}
	\end{equation}
	Since $\vect{x}=\vect{\Pi}^{(0)}\vect{x}^{(1)}=\whvect{x}^{(1)}$, \eqref{eqn:error_recursion} implies that
	\begin{equation}
		\|\vect{\Pi}^{(l)}\vect{x}^{(l+1)}-\whvect{x}^{(l+1)}\|_2 \leq \big(\sum_{k=1}^l (\tau+\epsilon)^{l-k}\big) \varepsilon B = \frac{(\tau+\varepsilon)^l-1}{\tau+\varepsilon-1} \varepsilon B.
	\end{equation}
	Specifically, for the output signals, $\vect{y}=\vect{x}^{L+1}$ and $\whvect{y}=\whvect{x}^{(L+1)}$, it implies that
	\begin{equation*}
		\|\whvect{y}-\vect{\Pi y}\|_2 \leq \varepsilon \frac{(\tau+\varepsilon)^L-1}{\tau+\varepsilon-1} B.
	\end{equation*}
\end{proof}
Therefore, if the hyperparameters of RePurpose are chosen carefully, we can ensure that the output of the modified neural network is close to the original model (after accounting for the possible rearrangement of the neurons of the output layer).

\section{Proofs of the Main Results}
\subsection{Proof of Lemma 1}
\begin{lemma*}
	The solution of 
\begin{equation}
	\min_{\vect{x}} \|\vect{y}-\vect{x}\|_2^2 + \eta_1 \|\vect{x}\|_0 + \eta_2 \|\vect{x}_{\backslash j}\|_0,
	\label{eqn:opt_problem}
\end{equation}
	is given by element-wise hard-thresholding $\vect{y}$, i.e.,
	\begin{equation}
	x_n = \left\{
	\begin{array}{ll}
	0 & \textup{if }|y_n| \leq \sqrt{\eta} \\
	y_n & \textup{o.w.}
	\end{array}
	\right.
	\end{equation}
	where $\eta=\eta_1$ or $\eta_1+\eta_2$, depending on whether neuron $n$ is in $\vect{y}_{\backslash j}$ or not.
\end{lemma*} 
\begin{proof}
	 Let $\Omega$ be the indexes in the $j$-th block. Therefore, $\vect{x}_{\backslash j}$ consists of elements of $\vect{x}$ that are not in the set $\Omega$, and
	\begin{align*}
		\|\vect{y}-\vect{x}\|_2^2 + \eta_1 \|\vect{x}\|_0 + \eta_2 \|\vect{x}_{\backslash j}\|_0 &= \sum_{n\in\Omega} (y_n-x_n)^2 + \eta_1 \Ind{x_n \neq 0} \\
		&+ \sum_{n\notin \Omega} (y_n-x_n)^2 + (\eta_1+\eta_2) \Ind{x_n \neq 0},
	\end{align*}
	where $\Ind{z}=1$ if $z$ is true and is $0$ otherwise. Therefore, the minimization in \eqref{eqn:opt_problem} can be cast as separate minimizations over scalars $x_n$. For example, if $n\in\Omega$, there are two possibilities for $x_n$,
	\begin{equation*}
	\left \{ 
	\begin{array}{ll}
	x_n = 0 & \Rightarrow cost = y_n^2 \\
	x_n\neq 0 & \Rightarrow cost = \min_{x_n\neq 0} (y_n-x_n)^2 + \eta_1 = \eta_1
	\end{array}
	\right.
	\end{equation*}
	Hence, the solution would be 
	\begin{equation*}
		n\in \Omega:\quad x_n^* = \left\{
		\begin{array}{ll}
			0 & \textup{if } |y_n|\leq \sqrt{\eta_1} \\
			y_n & \textup{o.w.}
		\end{array}
		\right.
	\end{equation*}
	
	Similarly, 
	\begin{equation*}
		n\notin \Omega:\quad x_n^* = \left\{
		\begin{array}{ll}
		0 & \textup{if } |y_n|\leq \sqrt{\eta_1+\eta_2} \\
		y_n & \textup{o.w.}
		\end{array}
		\right.
	\end{equation*}
\end{proof}

\subsection{Proof of Theorem 2}
\begin{theorem*}
	Algorithm 1 finds the optimum solution of 
	\begin{equation}
		\min_{\whvect{W},\vect{\Pi}}\|\whvect{W}-\vect{W\Pi}^\Tr\|_F^2 + \eta_1 \|\whvect{W}\|_0 + \eta_2\|\vect{M}\odot\whvect{W}\|_0,
	\end{equation}
	with time complexity $\BigO(N^3)$, where $N$ is the number of layer's neurons (number of columns of $\vect{W}$).
\end{theorem*}
\begin{proof}
	First, we note that for any permutation matrix $\vect{\Pi}$, $\|\whvect{W}-\vect{W\Pi}^\Tr\|_F^2 = \|\whvect{W}\vect{\Pi}-\vect{W}\|_F^2$, $\|\whvect{W}\|_0 = \|\whvect{W}\vect{\Pi}\|_0$, and $\|\vect{M}\odot\whvect{W}\|_0 = \|(\vect{M\Pi})\odot(\whvect{W}\vect{\Pi})\|_0$. Therefore, by defining $\vect{X}=\whvect{W}\vect{\Pi}$, the optimization \eqref{eqn:Lagrangian_RePurposeP_OPT} can be rewritten as
	\begin{align*}
		&\min_{\vect{\Pi}} \min_{\vect{X}} \|\vect{X}-\vect{W}\|_F^2 + \eta_1 \|\vect{X}\|_0 + \eta_2 \|(\vect{M\Pi})\odot\vect{X}\|_0 \\
		=&\min_{\vect{\Pi}} \min_{\vect{X}} \sum_{{i,k:\, \Pi_{k,i}=1}} \|\vect{x}_i-\vect{w}_i\|_2^2 + \eta_1 \|\vect{x}_i\|_0 + \eta_2 \|\vect{m}_k\odot \vect{x}_i\|_0, \\
		=&\min_{\vect{\Pi}} \sum_{{i,k:\, \Pi_{k,i}=1}} \min_{\vect{x}_i}\|\vect{x}_i-\vect{w}_i\|_2^2 + \eta_1 \|\vect{x}_i\|_0 + \eta_2 \|\vect{m}_k\odot \vect{x}_i\|_0.
	\end{align*}
	On the other hand, recall that $\vect{M}=1-\diag{\vect{1}_{\iota_{_1}\times n_{_1}},\ldots,\vect{1}_{\iota_{_P}\times n_{_P}}}$, and hence if $\vect{m}_k$ is from the $j$-th sub-block, i.e., it corresponds to the $j$-th worker, the inner minimization would be 
	\begin{equation}
		C_{ji}=\min_{\vect{x}_i}\|\vect{x}_i-\vect{w}_i\|_2^2 + \eta_1 \|\vect{x}_i\|_0 + \eta_2 \|\vect{x}_{i,\backslash j}\|_0.
		\label{eqn_cji}
	\end{equation}
	By repeating the $k$-th row of matrix $\vect{C}$ whose elements are defined as \eqref{eqn_cji} to construct the new $N\times N$ matrix $\wtvect{C}$, we will have $C_{ji}=\tilde{C}_{ki}$. Therefore, 
	\begin{align*}
		\min_{\whvect{W},\vect{\Pi}}\|\whvect{W}-\vect{W\Pi}^\Tr\|_F^2 + \eta_1 \|\whvect{W}\|_0 + \eta_2\|\vect{M}\odot\whvect{W}\|_0 = 
		\min_{\vect{\Pi}} \sum_{{(i,k):\, \Pi_{k,i}=1}} \tilde{C}_{ki}.
	\end{align*}
	As a result, selecting the best neuron assignment boils down to choosing $N$ elements from $\wtvect{C}$ such that from each row or column only one element is selected and the sum of the selected values is minimum. This problem can be solved efficiently in polynomial time using the Hungarian algorithm. \cite{Tomizawa1971Techniques,Jonker1987} solve the assignment algorithm with $\BigO(N^3)$ time complexity. Since the complexity of creating $\wtvect{C}$ is at most $\BigO(N^2)$, the total complexity of Algorithm 1 would be $\BigO(N^3)$.
\end{proof}

\section{Reduction in Computational Complexity}
\label{sec:comp-complexity}
One major benefit of applying RePurpose, as demonstrated in simulations, is the reduction in the computational complexity. For the sake of simplicity, assume that there are $P=2$ workers. Recall that the computations at worker $1$ is given as $\vect{y}_1=\vect{W}_{11}^\Tr \vect{x}_1+\vect{b}_1+\vect{W}_{12}^\Tr\vect{x}_2$. By the application of RePurpose to the weight matrix $\vect{W}$, the off-diagonal blocks, $\vect{W}_{12}$ and $\vect{W}_{21}$, become sparse. Let $\Omega$ be the indexes of the columns of $\vect{W}_{12}$ which are non-zero, and define $\wtvect{W}_{12}$ to be the restriction of $\vect{W}_{12}$ to those non-zero columns. Similarly, define $\wtvect{x}_2$ to be the restriction of $\vect{x}_2$ to the indexes given by $\Omega$. Therefore, $\vect{y}_1$ can be more efficiently calculated as $\vect{y}_1=\vect{W}_{11}^\Tr \vect{x}_1+\vect{b}_1+\wtvect{W}_{12}^\Tr\wtvect{x}_2$. If $\vect{W}_{12}$ is an $m\times n$ matrix, the computational complexity and the communication requirement of the cross-term $\vect{W}_{12}^\Tr\vect{x}_2$ in the original calculation would be $\BigO(mn)$ and $\BigO(m)$, respectively. RePurpose reduces these complexities to $\BigO(|\Omega|n)$ and $\BigO(|\Omega|)$. As shown in simulations, the set $\Omega$ can be extremely small, making the computational complexity of the cross-term negligible. For example, in applying the proposed technique to an $N \times N$ matrix to distributed its computations over $2$ workers, if the number of cross dependencies are reduced by a factor of $10$, then the computational complexity of matrix multiplication would be reduced to $0.275\,N^2$ per worker, almost $1.8$ times reduction from $N^2/2$ in naive parallel implementation.

\section{Extension of RePurpose to Convolutional Layers}
Consider a convolutional layer whose input consists of $c_{in}$ channels of $d$-dimensional tensors and its output has $c_{out}$ channels. Let $h(z_0,\ldots,z_{d-1},c_{in}, c_{out})$ be the kernel. For the sake of simplicity in notations, we ignore strides and dilation in convolution operator. Hence, the output would be
\begin{equation*}
	O(x_0,\ldots,x_{d-1},k) = \sum_{l=1}^{c_{in}} \sum_{z_0,\ldots,z_{d-1}} h(z_0,\ldots,z_{d-1},l, k)\, I(x_0+z_0,\ldots,x_{d-1}+z_{d-1}, l),
\end{equation*} 
where $I(\cdot)$ is the input $d$-dimensional tensor with $c_{in}$ channels and $O(\cdot)$ is the output tensor. 

Note that due to the nature of the convolution operator, it is not possible to rearrange the neurons within each channel (e.g., changing locations of pixels in images). However, we propose to change the order of the channels. Note that the convolution can be rewritten as
\begin{equation*}
	O_k(x_0,\ldots,x_{d-1}) = \sum_{l=1}^{c_{in}} h_{l,k}*I_l\,(x_0,\ldots,x_{d-1}),
\end{equation*} 
where $h_{l,k}(\cdots)=h(\cdots,l,k)$ is the kernel connecting input channel $l$ to output channel $k$, $I_l(\cdot)$ is the $l$-th channel of the input tensor, and $O_k(\cdot)$ is the $k$-th output channel. Now, similar to \eqref{eqn:PRP_param_cost}, we can define the cost of assigning \emph{channel} $i$ to the $j$-th worker as follows:
\begin{equation}
	C_{ji} = \min_{\{\whvect{h}_{l,i}\}} \sum_{l=1}^{c_{in}} \|\vect{h}_{l,i}-\whvect{h}_{l,i}\|_F^2 + \eta_1 \sum_{l=1}^{c_{in}} \Ind{\whvect{h}_{l,i}\neq \vect{0}} + \eta_2 \sum_{l: l\notin\set{C}_{in}(j)} \Ind{\whvect{h}_{l,i}\neq \vect{0}},
	\label{eqn:PRP_conv_cost}
\end{equation}
where $\set{C}_{in}(j)$ is the set of input channels located at the $j$-th worker, and $\Ind{z}=1$ if $z$ is true, and is $0$, otherwise. Note that for the convolutional layers, we treat the individual filters as a whole, and the entire channel filter may be set to zero, not the individual coefficients. The solution of \eqref{eqn:PRP_conv_cost} is given by hard-thresholding,
\begin{equation}
	\whvect{h}_{l,i} = \left\{
		\begin{array}{ll}
			\vect{0} & \|\vect{h}_{l,i}\|_F^2 \leq \eta \\
			\vect{h}_{l,i} & \textup{o.w.}
		\end{array}
	\right.
\end{equation}
where $\eta=\eta_1$ if $l\in\set{C}_{in}(j)$ and $\eta=\eta_1+\eta_2$, otherwise.

With the new assignment cost, RePurpose for convolutional layers is simply given as in Alg.~\ref{alg:RePurposeP}.

\end{document}